\documentclass[11pt]{article}

\usepackage{subfigure}
\usepackage{booktabs} 
\usepackage{comment}

\newcommand{\imgPath}{.} 

\usepackage[margin=2cm]{geometry}
\newcommand{\tdotoggle}{0}
\usepackage{preamble}

\newcommand{\sqloss}{\ell_{\mathrm{sq}}}

\newcommand{\init}{{\rm init}}

\newcommand{\zero}{{\vec 0}}

\newcommand{\ind}{\mathbbm{1}}

\newcommand{\NTK}{\mathsf{NTK}}

\newcommand{\Pbsp}{\calP^{\mathsf{bsp}}}
\newcommand{\Plp}{\calP^{\mathsf{lp}}}

\begin{document}
\vspace*{-5mm}
\begin{center}
\hrule height 2.5pt
\vspace{3mm}{\LARGE Quantifying the Benefit of \\[2mm] Using Differentiable Learning
over Tangent Kernels}\\[3mm]
\hrule height 1.25pt
\vspace*{3mm}
\newcommand{\HUJI}{{\scriptsize Hebrew University of Jerusalem}}
\newcommand{\EPFL}{{\scriptsize EPFL}}
\newcommand{\TTIC}{{\scriptsize Toyota Technological Institute at Chicago}}
\begin{tabular*}{0.9\textwidth}{ll @{\extracolsep{\fill}} r}
{\bf Eran Malach} & \HUJI & {\tt eran.malach@mail.huji.ac.il} \\
{\bf Pritish Kamath} & \TTIC  & {\tt pritish@ttic.edu} \\
{\bf Emmanuel Abbe} & \EPFL & {\tt emmanuel.abbe@epfl.ch} \\
{\bf Nathan Srebro} & \TTIC & {\tt nati@ttic.edu}\\[1mm]
\multicolumn{3}{c}{{Collaboration on the Theoretical Foundations of Deep Learning} (\href{https://deepfoundations.ai/}{deepfoundations.ai})}
\end{tabular*}
\vspace{5mm}
\end{center}

\begin{abstract}
We study the relative power of learning with gradient descent on differentiable models, such as neural networks, versus using the corresponding tangent kernels.
We show that under certain conditions, gradient descent achieves small error only if a related tangent kernel method achieves a non-trivial advantage over random guessing (a.k.a.\ weak learning), though this advantage might be very small even when gradient descent can achieve arbitrarily high accuracy.
Complementing this, we show that without these conditions, gradient descent can in fact learn with small error even when no kernel method, in particular using the tangent kernel, can achieve a non-trivial advantage over random guessing.
\end{abstract}

\section{Introduction}\label{sec:intro}

A recent line of research seeks to understand Neural Networks through their kernel approximation, as given by the Neural Tangent Kernel \citep[NTK,][]{jacot18ntk}. The premise of the approach is that in certain regimes, the dynamics of training neural networks are essentially the same as those of its first order Taylor expansion at initialization, which in turn is captured by the Tangent Kernel at initialization.  It is then possible to obtain convergence, global optimality and even generalization guarantees by studying the behaviour of training using the Tangent Kernel \citep[e.g.][and many others]{li2018learning,
du2018provably,
chizat2018note,
zou20gradient,allen2018learning,arora2019fine,
du19global
}.
Some have also suggested using Tangent Kernel directly for training \citep[e.g.][]{arora19exact}, even suggesting it can sometimes outperform training by GD on the actual network \citep{Geiger_2020}.

Can all the success of deep learning be explained using the NTK?  This would imply that we can replace training by gradient descent on a non-convex model with a (potentially simpler
to train, and certainly better understood) kernel method.  Is anything learnable using gradient descent on a neural network or other differentiable model also learnable using a kernel method (i.e. using a kernelized linear model)?

This question was directly addressed by multiple authors, who showed examples where neural networks trained with gradient descent (or some variant thereof) provably outperform the best that can possibly be done using the tangent kernel or any linear or kernel method, under  different settings and assumptions~\citep{yehudai19power,allenzhu19resnets,allenzhu20backward,li20beyondntk,daniely20parities,ghorbani19linearized, ghorbani20when}. However in these examples, while training the model with gradient descent performs better than using the NTK, the error of the NTK is still much better then baseline, with a significant ``edge'' over random guessing or using a constant, or null, predictor.  That is, the NTK at least allows for ``weak learning''.  In fact, in some of the constructions, the process of leveraging the ``edge'' of a linear or kernel learner and amplifying it is fairly explicit \citep[e.g.][]{allenzhu19resnets,allenzhu20backward}. The question we ask in this paper is:

\begin{center}
{\em Can gradient descent training on the actual deep (non-convex) model only amplify (or ``boost'') the edge of the NTK, with the NTK \textbf{required} to have an edge in order to allow for learning in the first place?  Or can differentiable learning succeed even when the NTK ---or any other kernel--- does not have a non-trivial edge?}
\end{center}
Can gradient descent succeed at ``strong learning'' only when the NTK achieves ''weak learning''? Or is it possible for gradient descent to succeed even when the NTK is not able to achieve any significant edge?

\paragraph{Our Contributions.} The answer turns out to be subtle, and as we shall see, relies crucially on two important considerations: the {\em unbiasedness} of the initialization, and whether we can rely on {\it input distribution knowledge} for the initialization.

We start, in \autoref{sec:separation}, by providing our own example of a learning problem where Gradient Descent outperforms the NTK, and indeed any kernel method.  But unlike the previous recent separating examples, where the NTK enjoys a considerable edge (constant edge, or even near zero error, but with a slower rate than GD), we show that the edge of the NTK, and indeed of any (poly-sized) kernel, can be arbitrarily close to zero while the edge of GD can be arbitrarily large. The edge of the NTK in this example is nonetheless vanishing at low rate, i.e., polynomially, and this leads us to ask whether the NTK must have at least a polynomial edge when GD succeeds.

In \autoref{thm:unbiased} of \autoref{subsec:unbiased-init} we show that when using an {\em unbiased} initialization, that is, where the output of the model is $0$ at initialization, then indeed the NTK must have a non-trivial edge (polynomial in the accuracy and scale of the model) in order for gradient descent \removed{on a differentiable model} to succeed.

The requirement that the initialization is unbiased turns out to be essential:  In \autoref{sec:no-weak-learn} we show an example where gradient descent on a model succeeds, from a biased initialization, even though {\em no} (reasonably sized) kernel method (let alone the NTK) can ensure a non-trivial edge.

But all is not lost with biased initialization.  In \autoref{thm:dist-dependent} of \autoref{subsec:dist-dependent} we show that at least for the square loss, if gradient descent succeeds from any (possibly biased) initialization, then
we can construct some alternate random ``initialization'' such that the Tangent Kernel at this random initialization (i.e.~in expectation over the randomness) {\em does} ensure a non-trivial edge.
Importantly, this random initialization must depend on knowledge of the input distribution.
Again, our example in \autoref{sec:no-weak-learn} shows that this distribution-dependence is essential.
This also implies a separation between problems learnable by gradient descent using an unbiased vs arbitrary initialization, emphasizing that the initialization being unbiased should not be taken for granted.

This subtle answer that we uncover is mapped in \Cref{tab:results}.\info{The columns corresponding to kernels are in increasing order of generality from left to right. Lower bound on edge holds also for columns on the right; Upper bound on edge holds also for columns on the left.}

\newcommand{\MR}[2]{\multirow{#1}{*}{#2}}
\newcommand{\MC}[3]{\multicolumn{#1}{#2}{#3}}
\newcommand{\ssl}[1]{\shortstack[l]{#1}}
\newcommand{\ssc}[1]{\shortstack[c]{#1}}
\newcommand{\aligncc}{\renewcommand\cellalign{cc}}
\newcommand{\alignlc}{\renewcommand\cellalign{lc}}

\newcommand{\goodnews}[1]{\textcolor{OliveGreen}{#1}}
\newcommand{\okaynews}[1]{\textcolor{RedOrange}{#1}}
\newcommand{\badnews}[1]{\textcolor{Gred}{#1}}

\newcommand{\mybullet}{{\scriptsize $\blacktriangleright$\ \xspace}}
\begin{table}[t]
\centering \small
\renewcommand{\arraystretch}{1.7}
\setlength\tabcolsep{4pt}
\begin{tabular}{!{\vrule width 1.1pt}l|l!{\vrule width 1.1pt}l|l|c!{\vrule width 1.1pt}}
\noalign{\hrule height 1.1pt}
\MC{2}{!{\vrule width 1.1pt}l!{\vrule width 1.1pt}}{\ssl{\\[2mm] \mbox{}}}
&
\MC{1}{c|}{\ssc{NTK at same\\[.5mm] Initialization\\[1mm] \mbox{}}}
&
\MC{1}{c|}{\ssc{\\[2mm]NTK at alternate\\[.5mm] randomized Initialization\\[1mm] \mbox{}}}
&
\MC{1}{c!{\vrule width 1.1pt}}{\ssc{NTK of arbitrary model \\ or even an arbitrary Kernel \\[0mm] \mbox{}}} \\
\noalign{\hrule height 1.1pt}
\MC{2}{!{\vrule width 1.1pt}l!{\vrule width 1.1pt}}{\ssl{\\[1mm] GD with unbiased \\ initialization \\ ($\forall_x f_{\theta_0}(x)=0$) \\ ensures small error\\[2mm] \mbox{} }} &
\ssl{\\\goodnews{\mybullet NTK edge $\ge \poly^{-1}$} \\ \quad (\myref{Thm.}{thm:unbiased})\\ \okaynews{\mybullet NTK edge can be $< \poly^{-1}$} \\ \phantom{\mybullet} \okaynews{while GD reaches 0 loss}  \\ \quad (\autoref{sep:1})}
&
\MC{2}{c!{\vrule width 1.1pt}}{\ssc{\okaynews{Edge with any kernel can be $< \poly^{-1}$} \\ \okaynews{while GD reaches 0 loss} \\(\autoref{sep:2})\\[1mm] \mbox{}}} \\
\hline
\multirow{2}{*}{\ssl{\\[-9mm]GD with \\arbitrary\\init.\\ensures\\small\\error}}
&
\ssl{Kernel (or\\alt init)\\can depend\\on input\\dist.~$\calD_\calX$} & \MC{1}{c|}{\multirow{2}{*}{\ssc{\badnews{NTK edge can be $= 0$} \\
\badnews{while GD reaches arb.\ low loss}  \\(\autoref{sep:3})}}}
&
\ssl{\\\goodnews{\mybullet NTK edge $\ge \poly^{-1}$} \\ \quad (\myref{Thm.}{thm:dist-dependent})\\ \okaynews{\mybullet NTK edge can be $< \poly^{-1}$} \\ \phantom{\mybullet} \okaynews{while GD reaches 0 loss} \\ \quad (\autoref{sep:2})}
&
\ssc{\\ \okaynews{Edge can be $< \poly^{-1}$} \\ \okaynews{while GD reaches 0 loss} \\ (\autoref{sep:2}) \\[1mm] \mbox{}}
\\
\cline{2-2}\cline{4-5}
& \ssl{\\[1mm]Dist-indep\\ kernels\\[1mm] \mbox{}}
& &
\MC{2}{c!{\vrule width 1.1pt}}{\ssc{\\[1mm] \badnews{edge with any kernel can be $< \exp^{-1}$}\\\badnews{while GD reaches arb.\ low loss} \\  (\autoref{sep:4})\\[0mm] \mbox{}}}  \\
\noalign{\hrule height 1.1pt}
\end{tabular}
\caption{Our Results at a glance: What does learnability with Gradient Descent imply, and does not imply, on the edge over null prediction that the Neural Tangent Kernel (NTK), or some other kernel, must have?  The results provide a complete picture (up to polynomial differences) of how large an edge is ensured in each scenario.}
\label{tab:results}
\end{table}

\section{Differentiable Learning and Tangent Kernels}\label{sec:prelims}
We consider learning a predictor $f:\calX \to \bbR$ over an {\em input space} $\calX$, so as to minimize its {\em population loss} $\calL_\calD(f) := \mean{(x,y) \sim \calD}{\ell(f(x), y)}$ with respect to a {\em source distribution} $\calD$ over $\calX \times \calY$, where \removed{$\calY$ a {\em label space} and} $\ell:\bbR \times \calY \to \bbR_{\ge 0}$ is a loss function.  We denote by $\ell'$ and $\ell''$ the derivatives of $\ell$ w.r.t.~its first argument.  Some of our guarantees apply to any smooth loss, i.e., $\abs{\ell''}:=\sup_{\what{y},y}\abs{\ell''(\what{y},y)}<\infty$, while others are specific to the square loss $\sqloss(\what{y},y)=\frac{1}{2}(\what{y}-y)^2$ with binary labels $\calY = \sbit$.  Our lower bounds and separation results are all stated for the square loss with binary labels (they could be adapted to other loss functions, but in order to demonstrate a separation, it is sufficient to provide a specific example).  To understand whether a predictor gives any benefit, we discuss the error of a predictor compared to that of ``null prediction'', i.e.~predicting $0$ on all inputs.  For the square loss with binary labels, the error of null prediction is $\calL_{\calD}(0)=0.5$, and  so we refer to the improvement $0.5-\calL_\calD(f)$ as the ``edge'' of the predictor $f$. That is, if $\calL_\calD(f)=0.5-\gamma$, we say that $f$ has an edge of $\gamma$ (over null prediction).

\paragraph{Differentiable Learning.}\label{setup}
We study learning by (approximate) gradient descent on a differentiable model, such as a neural network or other parametric model.  More formally, a {\em differentiable model} of size $p$ is a mapping $f:\bbR^p \times \calX \to \bbR$, denoted $f_{\theta}(x)$, where $\theta \in \bbR^p$ are the model parameters, or ``weights'', of the model, and the mapping is differentiable w.r.t. $\theta$.  We will control the scale of the model through a bound\footnote{The quantity $C_f$ mixes the scale of the gradients and of the function values.  We use a single quantity in order to minimize notation. Separating these would allow for more consistent scaling.}   $C^2_f := \sup_{\theta, x}(\norm{\nabla_{\theta} f_{\theta}(x)}_2^2+f_\theta^2(x))$ on the norm of the gradients and function values\footnote{%
	ReLU networks do not fit in this framework because (1) they are not differentiable; and (2) their gradients are not bounded.  The first issue is not significant, since we never rely on a bound on the second derivatives of the model (only the loss), and so we can either approximate the ReLU with an arbitrarily sharp approximation, or use some relaxed notion of a ``gradient'', as is actually done informally when discussing gradients of ReLU networks. Our results do rely on a bound on the gradient norms, which for ReLU networks depends on the scale of the weights $\theta$ (formally, it is sufficient to consider the model restricted to some large enough norm ball of weights which we will never exit).
	By truncating the model outside the scale of weights we would need to consider (in the analysis only), it is sufficient to take the supremum in the definition of $C_f$ only over weight settings explored by gradient descent, and so for ReLU networks $C_f$ would scale with this scale of $\theta$.%
}.

For a differentiable model $f_\theta(x)$ and an {\em initialization} $\theta_0\in\bbR^p$, gradient accuracy $\tau$, and stepsize\footnote{The stepsize doesn't play an important role in our analysis. We can also allow variable or adaptive stepsize sequences---for simplicity of presentation we stick with a fixed stepsize.} $\eta$, {\bf\boldmath $\tau$-approximate gradient descent training} is given by iterates of the form:
\begin{align}
    \theta^{(0)} &~\gets~ \theta_0 \notag\\
    \theta^{(t+1)} &~\gets~ \theta^{(t)} - \eta g_{t} \label{eq:gd} \\
    \omit\rlap{\hspace{-7mm}for $\norm{g_t - \nabla_\theta \calL_\calD(f_{\theta^{(t)}})}_2 \leq \tau$} \label{eq:gt} 
\end{align}
The gradient estimates $g_t$ can be obtained by computing the gradient on an empirical sample of $m$ training examples drawn from $\calD$, in which case we can ensure accuracy $\tau\propto  C_f/\sqrt{m}$.  And so, we can think of $C^2_f/\tau^2$ as capturing the sample size, and when we refer to the relative accuracy $\tau/C_f$ being polynomial, one might think of the sample size used to estimate the gradients being polynomial. But here we only assume the gradient estimates $g_t$ are good approximations of the true gradients of the population loss, and do not specifically refer to sampling.\footnote{In some regimes, one should in fact distinguish the setting of large sample sets and approximate population gradients in view of the universality result proved for SGD in \cite{abbe2020polytime,abbe2020universality}. We focus here on the approximate setting that better reflects the noisier regime; see discussion in \cite{abbe2020polytime} for  parities.}  We say that $\tau$-approximate gradient descent with model $f_\theta$, initialization $\theta_0$, gradient accuracy $\tau$ and $T$ steps ensures error $\eps$ on a distribution $\calD$ if with any gradient estimates satisfying \eqref{eq:gt}, the gradient descent iterates \eqref{eq:gd} lead to an iterate $\theta^{(T)}$ with population loss $\calL_\calD(f_{\theta^{(T)}}) \leq \eps$.

\paragraph{The Tangent Kernel.}
Consider the first order Taylor expansion of a differentiable model about some $\theta_*\in\bbR^p$:
\begin{align}
    f_{\theta}(x) &\approx f_{\theta_*}(x) + (\theta-\theta_*) \nabla_{\theta} f_{\theta_*}(x) \label{eq:taylor} \\
    & = \inangle{ [ 1 , \theta-\theta_*] , \phi_{\theta_*}(x) } \label{eq:linmodel}\\
    &\!\!\! \text{where } \phi_{\theta_*}(x) = [ f_{\theta_*}(x), \nabla_{\theta} f_{\theta_*}(x) ] \in \bbR^{p+1}\,,\label{eq:phiNTK}
\end{align}
which corresponds to a linear model with the feature map as in \eqref{eq:phiNTK}, and thus can also be represented using the kernel:
\begin{equation}\notag
\NTK_{\theta_*}^f(x,x')= f_{\theta_*}(x)f_{\theta_*}(x') + \inangle{\nabla_{\theta} f_{\theta_*}(x), \nabla_{\theta} f_{\theta_*}(x')}.
\end{equation}
In some regimes or model limits \citep[e.g.][]{daniely16toward,jacot18ntk,chizat19lazy},\nati{would be good to add ref before Jacot, I think Daniely and/or Du}\info{PK: Cited a paper of Daniely et al. that Eran pointed out. Anything else?} the approximation \eqref{eq:taylor} about the initialization remains valid throughout training, and so gradient descent on the actual model $f_\theta(x)$ can be approximated by gradient descent on the linear model \eqref{eq:linmodel}, i.e.~a kernalized linear model specified by the tangent kernel $\NTK_{\theta_*}^f$ at initialization.  One can then consider analyzing differentiable learning with the model $f$ using this NTK approximation, or even replacing gradient descent on $f$ with just using the NTK.\nati{Perhaps include references?  We have them in the intro.  Can move them here, or keep in intro.  Not important.} How powerful can such an approximation be relative to the full power of differentiable learning?  Can we expect that anything learnable with differentiable learning is also learnable under the NTK approximation?

To understand the power of a kernalized linear model with some kernel $K$, we should consider predictors realizable with bounded norm in the corresponding feature map (i.e.~norm balls in the corresponding Reproducing Kernel Hilbert Space):
\begin{align}
\calF(K,B)
&:=~ \set{x \mapsto \inangle{w, \phi(x)} : \norm{w}_2 \cdot R \le B} \label{eq:FKB}
\\
&\phantom{:}=~ \set{ f:\calX \to \bbR : \norm{f}_K \cdot R \le B } \label{eq:RKHSball}\\
\text{where } R &:=~ \sup_x \norm{\phi(x)}_2 ~=~ \sup_x \sqrt{K(x,x)} \notag
\end{align}
where $K(x,x')=\inangle{\phi(x),\phi(x')}$ and $\norm{f}_K$ is the $K$-RKHS-norm\footnote{The direct definition \eqref{eq:FKB} is sufficient for our purposes, and so we refrain from getting into the definition of an RKHS and the RKHS norm.  See, e.g.~\citet{smola1998learning}, for a definition and discussion. In \eqref{eq:RKHSball} we take the norm of $f$ to be infinite if $f$ is not in the RKHS.}  of $f$. Predictors in $\calF(K,B)$ can be learned to within error $\eps$ with $O(B^2/\eps^2)$ samples using $O(B^2/\eps^2)$ steps of gradient descent on the kernalized linear model. And since any predictor learned in this way will be in $\calF(K,B)$, showing that there is no low-error predictor in $\calF(K,B)$ establishes the limits of what can be learned using $K$.  We thus study the norm ball of the Tangent Kernel, which, slightly overloading notations, we denote:
\begin{equation}
\NTK^{f}_{\theta}(B) := \calF(\NTK^{f}_{\theta},B) \notag
\end{equation}

\paragraph{Learning Problems.}
For a fixed source distribution $\calD$, there is always a trivial procedure for ``learning'' it, where a good predictor specific to $\calD$ is hard-coded into the ``procedure''.  E.g.,~we can use a kernel with a single feature corresponding to this hard coded predictor.  Instead, in order to capture learning as inferring from data, and be able to state when this is {\em not} possible using some class of methods, e.g.~using kernel methods, we refer to a ``learning problem'' as a family $\calP$ of source distributions over $\calX \times \calY$, and say that a learning procedure learns the problem to within error $\eps$ if for any distribution $\calD \in \calP$ the method learns a predictor with population error at most $\eps$.

We consider learning both when the {\em input distribution}, i.e.~the marginal distribution $\calD_\calX$ over $\calX$, is known (or fixed) and when it is unknown.  {\bf Distribution dependent learning} refers to learning when the input distribution is either fixed (i.e.~all distributions $\calD\in\calP$ have the same marginal $\calD_\calX$), or when the model or kernel is allowed to be chosen based on the input distribution $\calD_\calX$, as might be the case when using unlabeled examples to construct a kernel.  In {\bf distribution independent learning}, we seek a single model, or kernel, that ensures small error for all source distributions in $\calP$, even though they might have different marginals $\calD_\calX$.\\[-3mm]

\noindent {\em Remark.} Typically, we would have a probabilistic quantifier (e.g.~in expectation, or with high probability) over sampling the training set, but we define differentiable learning, and learning using a kernel, without any probabilistic quantifier: For differentiable learning, we required that for any $\calD\in\calP$, gradient descent yields error at most $\eps$ using any sequence of gradient estimates satisfying \eqref{eq:gt} (this happens with high probability if we use $m=O(B^2/\tau^2)$ to obtain the gradient estimates, but we do not get into this in our results).  For kernel learning, we require that for any $\calD\in\calP$ there exists a predictor in $\calF(K,B)$ with error at most $\eps$, as these are the predictors that can be learned (with high probability) using the kernel method (but we do not get into the exact learning procedure).

\section{Gradient Descent Outperforms the NTK}\label{sec:separation}

In this section, we exhibit a simple example of a learning problem for which (i) approximate gradient descent on a differentiable model of size $p$ ensures arbitrarily small (in fact zero) error, whereas (ii) the tangent kernel for this model, or in fact any reasonably sized kernel, cannot ensure error better than  $0.5-\gamma$, for arbitrarily small $\gamma$, where recall $0.5$ is the error of the null prediction and $\gamma$ depends polynomially on the parameters of the model and of gradient descent.

Several authors have already demonstrated a range of examples where gradient descent ensures smaller error than can be ensured by any reasonably sized kernel, including the tangent kernel \citep{yehudai19power,ghorbani19limitations,ghorbani19linearized,allenzhu19resnets,allenzhu20backward,li20beyondntk,daniely20parities}. In \autoref{sec:survey} and \Cref{tab:prior-sep} we survey these papers in detail and summarize the separations they establish.  Here, we provide a concrete, self-contained and simple example for completeness, so that we can explicitly quantify the edge a kernel method might have. Our emphasis is on showing that the error for kernel methods is not just worse than gradient descent, but in fact not much better than null prediction---this is in contrast to prior separations, where the error possible using a kernel is either some constant between the error of null prediction and zero error, or more frequently, close to zero error, but not as close as the error attained by gradient descent (see \autoref{sec:survey} for details).  In our example, we also pay attention to whether the model's predictions at initialization are zero---a property that, as we will later see, plays a crucial role in our analysis.

\paragraph{\boldmath The Learning Problem.} We consider the problem of learning $k$-sparse parities over $n$ biased bits, i.e.~where the marginal distribution of each bit (i.e.~coordinate) is non-uniform.  In order to easily obtain lower bounds for any kernel (not just the tangent kernel), we let the input distribution be a mixture of independent biased bits and a uniform distribution, so that we can argue that no kernel can do well on the uniform component, and hence bound its error also on the mixture.  The key is that when $k$ is moderate, up to~$k=O(\log n)$, due to the bits being biased, a linear predictor based on all the bits in the parity has enough of an edge to be detected. This does give the tangent kernel a small edge over a trivial predictor, but this is the best that can be done with the tangent kernel. However, in a differentiable model, once this linear predictor is picked up, its support can be leveraged to output a parity of these bits, instead of their sum, thereby obtaining zero error.

Formally, for integers $2\leq k \leq n$ and for $\alpha\in(0,1)$, we consider the ``biased sparse parities'' problem $\Pbsp[n,k,\alpha]$ over $\calX = \sbit^n$ and $\calY = \sbit$, and learning w.r.t.~the square loss.
The problem consists of $\binom{n}{k}$ distributions over $\calX \times \calY$, each corresponding to a subset $I \subseteq [n]$ with $|I| = k$. The distribution $\calD_I$ is defined by the following sampling procedure:

\begin{itemize}
\item Let $\calD_0$ be the uniform distribution over $\sbit^n$ and let $\calD_1$ be the product distribution over $\sbit^n$ with $\Ex{x_i} = \frac{1}{2}$.
Sample\footnote{The uniform component in the mixture is used to facilitate the establishment of a lower bound for kernels by directly relying on Lemma~5 in \cite{kamath20approximate}. One could also directly lower bound the kernel error of a biased input distribution without the uniform component, by going beyond a standard application of \cite{kamath20approximate}; see Remark \ref{remove_mix}.} $x \sim \calD_\calX := (1-\alpha) \calD_0 + \alpha \calD_1$.
\item Set $y \gets \chi_I(x) := \prod_{i \in I} x_i$, the parity over the subset $I$.
\end{itemize}

\paragraph{\boldmath The differentiable model.} To learn $\Pbsp[n,k,\alpha]$, we construct a differentiable model that is a combination of a linear predictor $\inangle{\theta,x}$ and a ``selected-parity'', which outputs the parity of the subset of bits indicated by the (essential) support of $\theta$, that is, $\prod_{\abs{\theta_i} \ge \nu}x_i$ (for a suitable threshold $\nu$).
Importantly, both use the same weights $\theta$ (see \autoref{fig:separation-net}).  The weak edge of the linear predictor allows gradient descent to learn a parameter vector $\theta$ such that $\set{i \mid \abs{\theta_i} \ge \nu }=I$, and once this is learned, the ``selected-parity'' kicks in and outputs a perfect predictor.

Formally, we construct a differentiable model $f_{\theta}(x)$ with $\theta \in \bbR^n$ (i.e. size $p=n$) that behaves as follows
\begin{equation}\label{eq:f-theta-sec3}
\forall \theta \in \inparen{\textstyle \insquare{-\frac2n,\frac{2}{n}} \cup \insquare{\frac{3}{n},\frac{5}{n}}}^n \quad : \quad
f_{\theta}(x)
\begin{cases}
	\approx 2 \inangle{\theta,x} & \text{if } \theta \approx 0\\
	= \prod_{i\,:\,\theta_i \ge \frac3n} x_i & \text{if } \exists i \,:\, \theta_i \ge \frac{3}{n}
\end{cases},
\end{equation}
where $\approx$ stands for the first order approximation of $f_{\theta}$ about $\theta=0$.
The behaviour \eqref{eq:f-theta-sec3} is the {\em only} property we need to show that approximate gradient descent learns $\Pbsp[n,k,\alpha]$: as formalized in \autoref{clm:gd-upper-bound-1}, a single step of gradient descent starting at $\theta^{(0)}=0$ will leave $\theta^{(1)}_i \in \insquare{-\frac2n,\frac{2}{n}}$ for $i\not\in I$, while increasing $\theta^{(1)}_i \geq \frac{3}{n}$ for $i\in I$, thus yielding the correct parity.  In what follows we show how to implement $f_{\theta}$ as a continuous differentiable model with scale $C_f=O(n)$, and furthermore how this can implemented as a feed-forward neural network with piece-wise quadratic sigmoidal activations:\\[-0.5cm]
\begin{center}
\begin{tikzpicture}
\node at (-7,0.65) {$\sigma(z) := \begin{cases}
	0 & z < 0 \\
	2z^2 & z \in [0,\frac{1}{2}] \\
	4z - 2z^2 - 1 & z \in [\frac{1}{2}, 1] \\
	1 & z > 1
\end{cases}$};
\def\scle{1.1}
\tikzstyle{myplot} = [scale=\scle, variable=\x, Gred, line width=1pt]
\draw[-latex,black!50] (-1.3*\scle, 0) -- (2.5*\scle, 0) node[right] {\footnotesize $z$};
\draw[-latex,black!50] (0, -0.3*\scle) -- (0, 1.3*\scle) node[above] {\footnotesize $\sigma(z)$};
\node[below left] at (0,0) {\tiny $0$};
\draw[dotted] (1*\scle,1*\scle) -- (0*\scle,1*\scle) node[left] {\tiny $1$};
\draw[dotted] (1*\scle,1*\scle) -- (1*\scle,0*\scle) node[below] {\tiny $1$};

\draw[myplot, domain=-1:0]  plot ({\x}, {0});
\draw[myplot, domain=0:0.5] plot ({\x}, {2*\x*\x});
\draw[myplot, domain=0.5:1] plot ({\x}, {-2*\x*\x + 4*\x - 1});
\draw[myplot, domain=1:2]   plot ({\x}, {1});
\end{tikzpicture}
\end{center}
\vspace{-2mm}
The model $f_{\theta}$, illustrated in \autoref{fig:separation-net}, is defined below (where $\theta \circ x := (\theta_1 x_1, \ldots, \theta_n x_n)$).
\begin{align}
f_\theta(x) &~:=~
\sigma_{-1,1}^{-1,1} \inparen{\inner{\theta, x} + \calG(\theta \circ x)} \label{eq:def-f-1}\\
\calG(z) &~:=~ \textstyle S(\xi(z)) \cdot \inparen{H(\xi(z)) - {\sum_i z_i}} \label{eq:def-f-2}\\
S(s) &~:=~ \textstyle 1-\prod_{i=1}^n (1-s_i^2)\label{eq:def-f-3}\\
H(s) &~:=~ \textstyle \prod_{i=1}^n (1+s_i-s_i^2)\label{eq:def-f-3.5}\\
\xi(z)_i &~:=~ \sigma(n z_i - 2) - \sigma(- n z_i - 2) \label{eq:def-f-4}
\end{align}

\begin{wrapfigure}{R}{0.32\textwidth}
\centering
\begin{tikzpicture}
\def\scle{0.9}
\tikzstyle{myplot} = [scale=\scle, variable=\x, Gred, line width=1pt]
\draw[-latex,black!50] (-3*\scle, 0) -- (3*\scle, 0) node[right] {\footnotesize $z$};
\draw[-latex,black!50] (0, -1.3*\scle) -- (0, 1.3*\scle) node[above] {\footnotesize $\xi(z)$};
\node[below left] at (0,0) {\tiny $0$};
\draw[dotted] (2*\scle,1*\scle) -- (0*\scle,1*\scle) node[left] {\tiny $1$};
\draw[dotted] (-2*\scle,-1*\scle) -- (0*\scle,-1*\scle) node[right] {\tiny $-1$};
\draw[myplot, domain=-1:0]  plot ({\x+1}, {0});
\draw[myplot, domain=0:0.25] plot ({\x+1}, {8*\x*\x});
\draw[myplot, domain=0.25:0.5] plot ({\x+1}, {-8*\x*\x + 8*\x - 1});
\draw[myplot, domain=0.5:1.3]   plot ({\x+1}, {1});
\draw[myplot, domain=-1:0]  plot ({-\x-1}, {0});
\draw[myplot, domain=0:0.25] plot ({-\x-1}, {-8*\x*\x});
\draw[myplot, domain=0.25:0.5] plot ({-\x-1}, {8*\x*\x - 8*\x + 1});
\draw[myplot, domain=0.5:1.3]   plot ({-\x-1}, {-1});

\draw (1*\scle,-0.05*\scle) edge (1*\scle,0.05*\scle);
\draw[dotted] (1.5*\scle,0*\scle) edge (1.5*\scle,1*\scle);
\draw (-1*\scle,-0.05*\scle) edge (-1*\scle,0.05*\scle);
\draw[dotted] (-1.5*\scle,-1*\scle) edge (-1.5*\scle,0*\scle);

\node[below] at (1*\scle, 0)     {\tiny $\frac{2}{n}$};
\node[below] at (1.5*\scle, 0)   {\tiny $\frac{3}{n}$};
\node[above] at (-1*\scle+0.1, 0)   {\tiny $-\frac{2}{n}$};
\node[above] at (-1.5*\scle-0.1, 0) {\tiny $-\frac{3}{n}$};
\end{tikzpicture}
\caption{\label{fig:xi(z)}Visualization of $\xi(z)_i$.}
\end{wrapfigure}
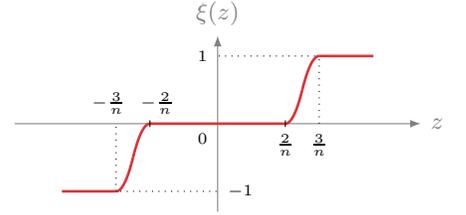

\noindent where $\sigma_{a,b}^{c,d}(z) := c + \sigma\left(\frac{z-a}{b-a} \right)(d-c)$, defined for all $a < b$ and $c, d$, satisfies (i) $\sigma_{a,b}^{c,d}(z) = c$ for every $z \le a$, (ii) $\sigma_{a,b}^{c,d}(z) = d$ for every $z \ge b$ and (iii) $\abs{\frac{d}{dz}\sigma_{a,b}^{c,d}(z)} \le \frac{2\abs{d-c}}{b-a}$. We visualize $\xi$, which is a (coordinate-wise) soft implementation of the $\sign(\cdot)$ function, in \autoref{fig:xi(z)}.

The intuition for $\calG$ is as follows. In the relevant regime of $\theta \in \inparen{[-\frac{2}{n},\frac{2}{n}] \cup [\frac{3}{n},\frac{5}{n}]}^n$ and any $x \in \sbit^n$ we have that $s = \xi(\theta \circ x) \in \set{-1,0,1}^n$, with $s_i = x_i$ if $\theta_i \in [\frac3n, \frac5n]$ and $s_i=0$ if $\theta_i \in [-\frac2n, \frac2n]$. In this regime of $\theta$, we have $S(s)= \ind\set{s \ne 0}$ and $H(s) = \prod_{i: s_i \ne 0} s_i$.
Thus, when $\theta_i \in [-\frac2n,\frac{2}{n}]$ for all $i$, we have $S(\xi(\theta\circ x)) = 0$ and hence $\calG(\theta \circ x) = 0$ for all $x \in \sbit^n$. On the other hand, when $\theta_i \in [\frac{3}{n}, \frac5n]$ for some $i$, we have $S(\xi(\theta\circ x)) = 1$ and hence $\calG(\theta \circ x) = \prod_{i\,:\,\theta_i \ge \frac{3}{n}} x_i - \inangle{\theta,x}$ for all $x \in \sbit^n$. This gives us \eqref{eq:f-theta-sec3}, by noting that $\sigma_{-1,1}^{-1,1}(z) \approx 2z$ at $z\approx 0$ (first order approximation) and $\sigma_{-1,1}^{-1,1}(z) = \sign(z)$ when $|z| \ge 1$.

Furthermore, we show how to implement $S$, $H$ (and hence $\mathcal{G}$) as a neural network using the activation $\sigma$. This is based on observing that we can compute squares in a bounded range by exploiting the quadratic part of the nonlinearity, i.e.,  $z^2=2r^2 \inparen{\sigma(z/2r)+\sigma(-z/2r)}$ for $z \in [-r,+r]$. The $n$-term products in \eqref{eq:def-f-3} and \eqref{eq:def-f-3.5} can be implemented with subnetworks of depth $O(\log n)$ and width $O(n)$ that recursively computes products of pairs, using $uv =\frac{1}{2} ((u+v)^2-u^2-v^2)$; if $u, v \in [-1,1]$, we only need to compute squares in the range $[-2,2]$. There is a slight issue in computing the final product because $H(\xi(\theta \circ x)) - \inangle{\theta,x}$ can be unbounded. However, $\inangle{\theta,x} \le 5$ in the relevant regime of $\theta$ and so we can correctly compute the final product in this regime; the product is not implemented correctly outside the relevant regime, but that doesn't affect the learning algorithm. The overall network thus has depth $O(\log n)$ and $O(n)$ edges.

We can also calculate that for any $i \in [n]$, $\theta$ and $x \in \sbit^n$, it holds that $\lvert \frac{\partial}{\partial\theta_i} f_{\theta}(x)\rvert \le O(n)$ and $|f_{\theta}(x)| \le 1$. Thus, we get that $C^2_f = \sup_{\theta, x} \|\nabla_{\theta} f_{\theta}(x)\|^2 + f_\theta(x)^2 \le O(n^2)$.

\begin{figure}
\begin{center}
\begin{tikzpicture}[scale=1,transform shape]
\tikzset{
	var/.style   = {circle, line width=1pt, draw=Ggreen, fill=Ggreen!20, minimum size=7mm, inner sep=2pt},
	theta/.style = {rectangle, line width=1pt, rounded corners=3pt, minimum height=4mm, minimum width=4mm, inner sep=0pt, draw=myGold, fill=myGold!20},
	gate/.style  = {circle, line width=1pt, draw=myBlue, fill=myBlue!20, inner sep=2pt},
	sig/.style   = {circle, line width=1pt, draw=myPurple, fill=myPurple!20, minimum size=5mm},
	mininode/.style   = {circle, line width=0.5pt, draw=black, fill=black!20, minimum size=1.7mm, inner sep=0pt, outer sep=0pt}
}

\def\xscle{1}
\def\xgap{0.75}
\node[var] (x1) at (-4*\xgap,0) {\small $x_1$};
\node[var] (x2) at (-2*\xgap,0) {\small $x_2$};
\node (xdots)   at (0,0) {\huge{$\cdots$}};
\node[var] (xn) at (2*\xgap,0) {\small $x_n$};

\def\ygap{1}
\def\yone{\ygap*1.3}
\node[theta] (t1) at (-4*\xgap,\yone) {};
\node[theta] (t2) at (-2*\xgap,\yone) {};
\node (tdots)     at (0*\xgap,\yone) {\huge{$\cdots$}};
\node[theta] (tn) at (2*\xgap,\yone) {};

\def\ytwo{3*\ygap}
\node[rectangle, rounded corners=3pt, draw=Gred, fill=Gred!20, minimum width=2.7cm, minimum height=1.2cm, line width=1pt] (G) at (-1.9,\ytwo) {\LARGE $\qquad\quad\,\calG$};

\def\xm{-0.1}
\def\ym{-0.4}
\node[mininode] (i1) at ($(G)+(\xm-0.2,\ym)$) {};
\node[mininode] (i2) at ($(G)+(\xm+0.2,\ym)$) {};
\def\ym{0}
\node[mininode] (g1) at ($(G)+(\xm-0.54,\ym)$) {};
\node[mininode] (g2) at ($(G)+(\xm-0.18,\ym)$) {};
\node[mininode] (g3) at ($(G)+(\xm+0.18,\ym)$) {};
\node[mininode] (g4) at ($(G)+(\xm+0.54,\ym)$) {};
\def\ym{0.4}
\node[mininode] (o)  at ($(G)+(\xm,\ym)$) {};

\path[-latex,line width=0.3pt]
(i1) edge (g1)
(i1) edge (g2)
(i1) edge (g3)
(i2) edge (g2)
(i2) edge (g3)
(i2) edge (g4)
(g1) edge (o)
(g2) edge (o)
(g3) edge (o)
(g4) edge (o);

\node[gate] (p1) at (1, \ytwo) {\small{\boldmath$+$}};
\def\ythree{4.3*\ygap}
\node[gate] (p2) at (-0.5, \ythree) {\small{\boldmath$+$}};
\def\yfour{5.4*\ygap}
\node[sig] (sig) at (-0.5, \yfour) {{\boldmath $\sigma$}};

\path[-{Stealth},line width=0.5pt]
(x1) edge node[midway, left] {{\boldmath $\theta_1$}} (t1)
(x2) edge node[midway, left] {{\boldmath $\theta_2$}} (t2)
(xn) edge node[midway, left] {{\boldmath $\theta_n$}} (tn)
(t1) edge[bend right=5] (G)
(t2) edge (G)
(tn) edge[bend left=15] (G)
(t1) edge[bend right=5] (p1)
(t2) edge[bend right=18] (p1)
(tn) edge (p1)
(G) edge (p2)
(p1) edge (p2)
(p2) edge (sig)
(sig) edge ($(sig.north)+(0,0.4)$);
\end{tikzpicture}
\end{center}
\caption{Schematic diagram for the construction of $f_\theta$, as used in \autoref{clm:gd-upper-bound-1}, with only trainable parameters being $\theta_1, \theta_2, \ldots, \theta_n$.  The sub-network $\calG$ is a fixed module implementing the ``selected-parity'' function.}
\label{fig:separation-net}
\end{figure}
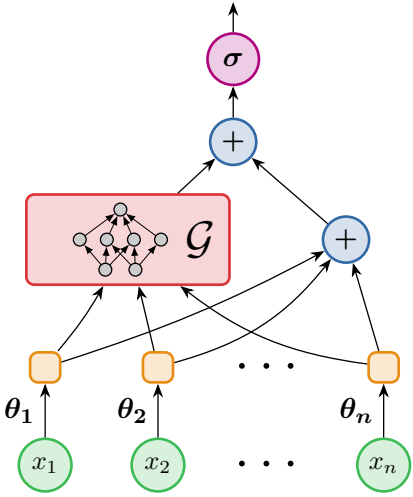

The following Claim (proved in \Cref{apx:neural-net-ub-1}) formalizes how a single step of gradient descent is sufficient for $\theta$ to be away from zero on $I$, and hence for the network to output the correct labels.

\begin{restatable}{claim}{gdubone}\label{clm:gd-upper-bound-1}
For any $n$, $k$ and $\alpha \in (0,1)$, and any $\calD_I \in \Pbsp[n,k,\alpha]$, $\tau$-approximate gradient descent on the model $f$ of size $p=n$ and $C_f=O(n)$ described above, with 
initialization $\theta_0 = 0$ (at which $\forall_x f_{\theta_0}(x)=0$), accuracy $\tau \le \alpha/2^{k}$, step size $\eta=2^k/(\alpha n)$ and $T=1$ step ensures $\calL_{\calD_I}(f_{\theta^{(T)}}) = 0$. In particular, if $k\leq\log n$ then an accuracy of $\tau \le \alpha/n$ is sufficient.
\end{restatable}

\noindent On the other hand, the next claim (proof in \Cref{apx:tang-kernel-lb-1}), establishes that the tangent kernel at initialization only has a small edge over the null prediction.

\begin{restatable}{claim}{ntkedge}\label{clm:ntk-edge-1}
The tangent kernel of the model $f$ at $\theta_0=0$ is the scaled linear kernel $\NTK^f_{\theta_0}(x,x')= 4 \inangle{x,x'}$, and for any $2\leq k \leq n$,  $\alpha\in(0,1)$ and $\calD_I\in\Pbsp[n,k,\alpha]$ the error with this kernel is $\inf_{h\in\NTK^f_{\theta_0}(B)} \calL_{\calD_I}(h) \geq \frac{1}{2} - \frac{\alpha}{2} = \frac12 - O\inparen{\frac{n^2 \tau}{C_f}}$ for all $B \ge 0$, where $\tau=\alpha/2^k$ as in \autoref{clm:gd-upper-bound-1}.
On the other hand, $\exists {h\in\NTK^f_{\theta_0}(n)}$ s.t. $\calL_{\calD_I}(h) \leq \frac{1}{2}-\frac{\alpha^2}{2^{2k}}=\frac{1}{2}-\Omega\inparen{\frac{n^2\tau^2}{C_f^2}}$
\end{restatable}
\noindent We already see that gradient descent can succeed where the tangent kernel at initialization can only ensure an arbitrarily small edge over the error achieved by null prediction, $\calL(0)=1/2$:
\begin{sepres}\label{sep:1}
For any $\gamma>0$, there exists a source distribution (with binary labels and w.r.t.~the square loss), such that
\begin{itemize}[leftmargin=*]
\item {[Gradient Descent]} Using a differentiable model with $p=2$ parameters, $T=1$ steps, and accuracy $\frac{\tau}{C_f}=\Theta(\gamma)$, $\tau$-approximate gradient descent can ensure zero squared-loss, but
\item {[Tangent Kernel]} The tangent kernel of the model at initialization does not ensure square loss lower than $\frac{1}{2}-\gamma=\frac{1}{2}-\Theta(\tau/C_f)$ (compare to the null prediction that always has square loss $\calL(0)=\frac{1}{2}$).
\end{itemize}
Furthermore, the gradient descent algorithm has an initialization $\theta_0$ s.t.~$\forall_x f_{\theta_0}(x)=0$.
\end{sepres}
\begin{proof} Apply Claims \ref{clm:gd-upper-bound-1} and \ref{clm:ntk-edge-1} to the sole source distribution in $\Pbsp[n=2,k=2,\alpha=2\gamma]$, (corresponding to $I=\{1,2\}$).
\end{proof}
\noindent Even though the tangent kernel at the initialization used by gradient descent might have an arbitrarily small edge, one might ask whether better error can be ensured by the tangent kernel at some other ``initialization'' $\theta$, or perhaps by the tangent kernel of some other model, or perhaps even some other kind of kernel\footnote{For each instance $\calD_I\in\Pbsp[n,k,\alpha]$, there is of course always a point $\theta$ at which the tangent kernel allows for prediction matching that of Gradient Descent, namely the iterate $\theta^{(T)}$ reached by Gradient Descent.  But to {\em learn} using a kernel, the kernel should be chosen without knowing the instance, i.e.~without knowing $I$.}. The following claim shows that this is not the case (proof in \Cref{apx:kernel-lb-2}).


\begin{restatable}{claim}{kerlbone}\label{clm:ker-lower-bound-1}
For all $\alpha < 1/2, k, p, B, n$, and any kernel $K$ corresponding to a $p$-dimensional feature map, there exists $\calD_I \in \Pbsp[n,k,\alpha]$ such that
\[\inf_{h \in \calF(K,B)}\ \calL_{\calD_I}(h) ~\ge~ \calL_{\calD_I}(0)-\frac{\alpha}{2}- \min\set{ \frac{p}{2|\Pbsp|},~ O\inparen{\frac{B^{\nicefrac{2}{3}}}{|\Pbsp|^{\nicefrac{1}{3}}}}}\]
where $\calL_{\calD_I}$ is w.r.t.~the square loss, and note that $|\Pbsp[n,k,\alpha]| = \binom{n}{k}$.
\end{restatable}
And so, setting $k=\Theta(\log n)$ and $\alpha=1/\poly(n)$, we see that gradient descent with polynomial parameters can learn $\Pbsp[n,k,\alpha]$, while no tangent kernel of a polynomial sized model (i.e.~with $p=\poly(n)$) can ensure better than an arbitrarily small polynomial edge:
\begin{sepres}\label{sep:2}
For any sequence $\gamma_n=1/\poly(n)$, there exists a sequence of learning problems  $\calP_n$ with fixed input distributions\removed{over $\calX_n=\left\{ \pm 1 \right\}^n$} (with binary labels and w.r.t.~the square loss,), such that 
\begin{itemize}[leftmargin=*]
\item {[Gradient Descent]} for each $n$, using a differentiable model with $p=n$ parameters, realizable by a neural network of depth $O(\log n)$ with $O(n)$ edges (where some edges have trainable weights and some do not), $T=1$ steps, polynomial accuracy $\frac{\tau}{C_f} = O\left(\frac{\gamma_n}{n^{2}}\right)$, and initialization $\theta_0$ s.t.~$\forall_x f_{\theta_0}(x)=0$, $\tau$-approximate gradient descent learns $\calP_n$ to zero loss; but
\item {[Poly-Sized Kernels]} no sequence of kernels $K_n$ corresponding to feature maps of dimension $\poly(n)$ (and hence no tangent kernel of a poly-sized differentiable models) can allow learning $\calP_n$ to square loss better than $\frac{1}{2}-\gamma_n$ for all $n$; and 
\item{[Arbitrary Kernel, Poly Norm]} no sequence of kernels $K_n$ of any dimension can allow learning $\calP_n$ to square loss better than $\frac{1}{2}-\gamma_n$ for all $n$ using predictors in $\calF(K_n,B_n)$ of norm $B_n = \poly(n)$.
\end{itemize}
\end{sepres}
\begin{proof}Consider $\calP_n=\Pbsp[n,k=\log_2 n,\alpha=\gamma_n=1/\poly(n)]$. Learnability using GD follows from \autoref{clm:gd-upper-bound-1}, noting that $\frac{\tau}{C_f}=\frac{\alpha}{n 2^k}=\frac{\gamma_n}{n^2}$ is inverse-polynomial in $n$.  Since $\binom{n}{\log_2 n}=n^{\Omega(\log n)}\gg\poly(n)$, if the kernel dimension $p(n)$ (or similarly, the norm $B_n$) is polynomial in $n$, then $\frac{p}{2\binom{n}{k}}=o(1)$, and so for large enough $n$, the error lower bound in \autoref{clm:ker-lower-bound-1} would be $> \frac{1}{2}-\gamma_n$. \end{proof}

\paragraph{Empirical Demonstration in Two-Layer Networks.}
While for ease of analysis we presented a fairly specific model, with many fixed (non-trainable) weights and only few trainable weights, we expect the same behaviour occurs also in more natural, but harder to analyze models.  To verify this, we trained a two-layer fully-connected ReLU network on the source distribution $\calD_{\alpha}$ analyzed above, for $n=128$ and $k=7$.  We observe that indeed when $\alpha>0$, and thus a linear predictor has at least some edge, gradient descent training succeeds in learning the sparse parity, while the best predictor in the Tangent Kernel cannot get error much better than $0.5$.  See \autoref{fig:parity_experiment} for details.

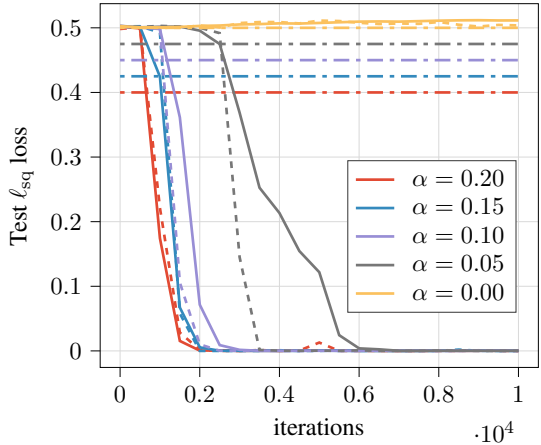
\begin{figure}
\centering
\scalebox{.85}{
\begin{tikzpicture}

\definecolor{color0}{rgb}{0.886274509803922,0.290196078431373,0.2}
\definecolor{color1}{rgb}{0.203921568627451,0.541176470588235,0.741176470588235}
\definecolor{color2}{rgb}{0.596078431372549,0.556862745098039,0.835294117647059}
\definecolor{color3}{rgb}{0.984313725490196,0.756862745098039,0.368627450980392}

\tikzset{
    mydotdash/.style={dash pattern=on 2pt off 3pt on 7pt off 3pt},
    mythick/.style={line width=1.2pt}
}

\begin{axis}[
axis background/.style={fill=white},
axis line style={black},
legend cell align={left},
legend style={fill opacity=1, draw opacity=1, text opacity=1, at={(0.95,0.36)}, anchor=east, draw=white!20!black, fill=white},
tick align=outside,
tick pos=left,
x grid style={black!15},
xlabel={iterations},
xmajorgrids,
xmin=-500, xmax=10500,
xtick style={black},
y grid style={black!15},
ylabel={Test $\sqloss$ loss},
ymajorgrids,
ymin=-0.02558925, ymax=0.53737425,
ytick style={black}
]
\addplot [mythick, color0]
table {%
0 0.5023
500 0.49904
1000 0.174365
1500 0.015385
2000 0.00049
2500 0.00076
3000 8.5e-05
3500 0.000235
4000 0.000145
4500 0.00048
5000 0.00074
5500 0.00081
6000 0.0004
6500 0.00019
7000 0.0001
7500 0.000175
8000 0.000665
8500 0.000125
9000 0.00029
9500 0.000825
10000 0.000205
};
\addlegendentry{$\alpha=0.20$}
\addplot [mythick, color0, dashed, forget plot]
table {%
0 0.4982
500 0.5012
1000 0.2212
1500 0.0283
2000 0.0017
2500 0
3000 0
3500 0.0013
4000 0
4500 0.0006
5000 0.0129
5500 0
6000 0
6500 0
7000 0
7500 0
8000 0
8500 0.0001
9000 0.0009
9500 0
10000 0
};
\addplot [mythick, color0, mydotdash, forget plot]
table {%
0 0.399999826671651
500 0.399999826671651
1000 0.399999826671651
1500 0.399999826671651
2000 0.399999826671651
2500 0.399999826671651
3000 0.399999826671651
3500 0.399999826671651
4000 0.399999826671651
4500 0.399999826671651
5000 0.399999826671651
5500 0.399999826671651
6000 0.399999826671651
6500 0.399999826671651
7000 0.399999826671651
7500 0.399999826671651
8000 0.399999826671651
8500 0.399999826671651
9000 0.399999826671651
9500 0.399999826671651
10000 0.399999826671651
};
\addplot [mythick, color1]
table {%
0 0.501565
500 0.5014
1000 0.424375
1500 0.06712
2000 0.005425
2500 0.00027
3000 0.00015
3500 0.00015
4000 7.5e-05
4500 0.000195
5000 0.000615
5500 0.00014
6000 7e-05
6500 0.00018
7000 7e-05
7500 0.00014
8000 0.00022
8500 0.000235
9000 0.000465
9500 0.000175
10000 0.00028
};
\addlegendentry{$\alpha=0.15$}
\addplot [mythick, color1, dashed, forget plot]
table {%
0 0.5022
500 0.5024
1000 0.4926
1500 0.0579
2000 0
2500 0
3000 0
3500 0
4000 0.001
4500 0
5000 0
5500 0.0008
6000 0
6500 0.0001
7000 0
7500 0.0001
8000 0
8500 0.002
9000 0.0002
9500 0
10000 0
};
\addplot [mythick, color1, mydotdash, forget plot]
table {%
0 0.424999826671651
500 0.424999826671651
1000 0.424999826671651
1500 0.424999826671651
2000 0.424999826671651
2500 0.424999826671651
3000 0.424999826671651
3500 0.424999826671651
4000 0.424999826671651
4500 0.424999826671651
5000 0.424999826671651
5500 0.424999826671651
6000 0.424999826671651
6500 0.424999826671651
7000 0.424999826671651
7500 0.424999826671651
8000 0.424999826671651
8500 0.424999826671651
9000 0.424999826671651
9500 0.424999826671651
10000 0.424999826671651
};
\addplot [mythick, color2]
table {%
0 0.50296
500 0.500335
1000 0.4982
1500 0.36145
2000 0.071505
2500 0.008975
3000 0.00163
3500 0.00018
4000 0.0002
4500 0.00018
5000 0.0002
5500 0.00022
6000 0.000195
6500 0.00031
7000 0.000305
7500 0.00034
8000 0.000155
8500 0.00017
9000 0.00016
9500 0.00045
10000 0.000125
};
\addlegendentry{$\alpha=0.10$}
\addplot [mythick, color2, dashed, forget plot]
table {%
0 0.501
500 0.5007
1000 0.4996
1500 0.1058
2000 0.0101
2500 0.0003
3000 0.0003
3500 0.0003
4000 0
4500 0
5000 0.0004
5500 0.0001
6000 0
6500 0.0011
7000 0.0002
7500 0
8000 0
8500 0.0001
9000 0.0007
9500 0
10000 0
};
\addplot [mythick, color2, mydotdash, forget plot]
table {%
0 0.449999826671651
500 0.449999826671651
1000 0.449999826671651
1500 0.449999826671651
2000 0.449999826671651
2500 0.449999826671651
3000 0.449999826671651
3500 0.449999826671651
4000 0.449999826671651
4500 0.449999826671651
5000 0.449999826671651
5500 0.449999826671651
6000 0.449999826671651
6500 0.449999826671651
7000 0.449999826671651
7500 0.449999826671651
8000 0.449999826671651
8500 0.449999826671651
9000 0.449999826671651
9500 0.449999826671651
10000 0.449999826671651
};
\addplot [mythick, white!46.6666666666667!black]
table {%
0 0.50201
500 0.500725
1000 0.501675
1500 0.501965
2000 0.49563
2500 0.47493
3000 0.36867
3500 0.252545
4000 0.21341
4500 0.15454
5000 0.121785
5500 0.024315
6000 0.00391
6500 0.002415
7000 0.000175
7500 0.00017
8000 0.00016
8500 0.00024
9000 0.00019
9500 0.000195
10000 0.00029
};
\addlegendentry{$\alpha=0.05$}
\addplot [mythick, white!46.6666666666667!black, dashed, forget plot]
table {%
0 0.5028
500 0.5013
1000 0.503
1500 0.502
2000 0.5001
2500 0.4918
3000 0.1451
3500 0.0014
4000 0.0004
4500 0.0001
5000 0.0005
5500 0.0001
6000 0.0009
6500 0.0001
7000 0.0001
7500 0
8000 0.0006
8500 0.0004
9000 0
9500 0
10000 0
};
\addplot [mythick, white!46.6666666666667!black, mydotdash, forget plot]
table {%
0 0.474999826671651
500 0.474999826671651
1000 0.474999826671651
1500 0.474999826671651
2000 0.474999826671651
2500 0.474999826671651
3000 0.474999826671651
3500 0.474999826671651
4000 0.474999826671651
4500 0.474999826671651
5000 0.474999826671651
5500 0.474999826671651
6000 0.474999826671651
6500 0.474999826671651
7000 0.474999826671651
7500 0.474999826671651
8000 0.474999826671651
8500 0.474999826671651
9000 0.474999826671651
9500 0.474999826671651
10000 0.474999826671651
};
\addplot [mythick, color3]
table {%
0 0.502605
500 0.500655
1000 0.50065
1500 0.50135
2000 0.502945
2500 0.50355
3000 0.505215
3500 0.50577
4000 0.50691
4500 0.506605
5000 0.50755
5500 0.508585
6000 0.509585
6500 0.50989
7000 0.509575
7500 0.51033
8000 0.511065
8500 0.51159
9000 0.511785
9500 0.511395
10000 0.511455
};
\addlegendentry{$\alpha=0.00$}
\addplot [mythick, color3, dashed, forget plot]
table {%
0 0.501
500 0.5008
1000 0.5001
1500 0.5001
2000 0.5036
2500 0.5005
3000 0.5074
3500 0.5089
4000 0.5089
4500 0.5068
5000 0.5112
5500 0.5105
6000 0.5078
6500 0.5059
7000 0.5067
7500 0.5062
8000 0.509
8500 0.5048
9000 0.502
9500 0.5045
10000 0.5037
};
\addplot [mythick, color3, mydotdash, forget plot]
table {%
0 0.499999826671651
500 0.499999826671651
1000 0.499999826671651
1500 0.499999826671651
2000 0.499999826671651
2500 0.499999826671651
3000 0.499999826671651
3500 0.499999826671651
4000 0.499999826671651
4500 0.499999826671651
5000 0.499999826671651
5500 0.499999826671651
6000 0.499999826671651
6500 0.499999826671651
7000 0.499999826671651
7500 0.499999826671651
8000 0.499999826671651
8500 0.499999826671651
9000 0.499999826671651
9500 0.499999826671651
10000 0.499999826671651
};
\end{axis}

\end{tikzpicture}}
\caption{Two-layer fully-connected ReLU network (128 hidden units) training on data sampled from $\calD_I$ with $n=128$, $k=7$ for different values of $\alpha$, trained using Adam optimizer with learning-rate of $0.01$. Solid lines show the test performance at each iteration averaged over 20 runs. Dashed lines show the test performance of the model with best test performance of the last iterate between the 20 runs. The horizontal dot-dash lines indicate the lower bound (from \autoref{clm:ker-lower-bound-1}) on best accuracy attainable by any kernel method with feature map with as many number of parameters as the NTK.}
\label{fig:parity_experiment}
\end{figure}



\section{Ensuring the Tangent Kernel has an Edge}\label{sec:weak-learn}


In the example of \autoref{sec:separation} we saw that the Tangent Kernel at initialization cannot ensure small error, and even though Gradient Descent finds a perfect predictor, the tanget kernel only has an arbitarily small edge over the null prediction.  But this edge isn't zero: the second part of \autoref{clm:ntk-edge-1} tells us that at least in the example of \autoref{sec:separation}, the tangent kernel {\em will} have an edge, and this edge is polynomial (even linear) in the accuracy required of gradient descent.  Is this always the case?  We now turn to asking whether whenever gradient descent succeeds in ensuring small error, then the Tangent Kernel must indeed have an edge that is at least polynomial in the gradient descent parameters.

\subsection{With Unbiased Initialization}
\label{subsec:unbiased-init}

We begin by considering situations where gradient descent succeeds starting at an initialization yielding zero predictions on all inputs, i.e.~such that:
\begin{equation}\label{eq:unbiased}
\forall\,x\ : \ f_{\theta_0}(x)~=~0.
\end{equation}
Following \citet{chizat19lazy}, we refer to $\theta_0$ satisfying \eqref{eq:unbiased} as {\em unbiased} initializations.  With an unbiased initialization, the Taylor expansion \eqref{eq:taylor} does not have the zero-order, or ``bias'' term, there is no need for the first coordinate in the feature vector $\phi_{\theta_0}$, and the Tangent Kernel becomes $\NTK_{\theta_0}^f(x,x')= \inangle{\nabla_{\theta} f_{\theta_0}(x), \nabla_{\theta} f_{\theta_0}(x')}$, simplifying the Neural Tangent Kernel analysis. \citet{chizat19lazy}, \citet{woodworth20kernelrich} and others discuss how to construct generic unbiased initializations, namely by including pairs of units that cancel each other out at initialization, but quickly diverge during training and allow for meaningful learning.  The initialization used in \autoref{clm:gd-upper-bound-1} of \autoref{sec:separation} also satisfies \eqref{eq:unbiased}.

In \autoref{thm:unbiased} below, we show that if gradient descent improves over the loss at initialization, then the Tangent Kernel must also have an edge over initialization, which for an unbiased initialization means an edge over the null prediction.

We can also considered a relaxation of  the unbiasedness assumption \eqref{eq:unbiased}, and instead require only that the output of the network at initialization is very close to zero.  This would be the case with some random initialization schemes which initialize the network randomly, but with small magnitude\nati{Would be good to cite the init where the magnitude is $\propto \sqrt{width}$.}.  As long as the deviation from unbiasedness is smaller than the edge guaranteed by \autoref{thm:unbiased}, the Theorem would still ensure the tangent kernel has an edge over null prediction.

\begin{theorem}\label{thm:unbiased}
For any smooth loss function $\ell$, differentiable model $f$, initialization $\theta_0$, and source distribution $\calD$, if $\tau$-accurate gradient descent for $T$ steps ensures error $\calL_{\calD}(f_{\theta^{(T)}}) \leq \eps$, then for $B = C_f+\frac{\tau}{\abs{\ell''}C_f}$, there exists $h \in \NTK^{f}_{\theta_0}(B)$ with 
\begin{equation}\label{eq:thmgen}
\textstyle \calL_{\calD}(h) \leq \max\left(\eps,\,\calL_\calD(f_{\theta_0}) - \frac{\tau^2}{2\abs{\ell''} C^2_f}\right).  
\end{equation}
In particular, if $\theta_0$ is an {\bf unbiased initialization} and $\eps<\calL_\calD(0)$, then
\begin{equation}\label{eq:thmun}
 \textstyle \calL_{\calD}(h) \leq \calL_\calD(0) - \frac{\tau^2}{2\abs{\ell''} C^2_f}.  
\end{equation}
\removed{
loss function $\ell$, a distribution $\calD \in \Delta_{\calX \times \calY}$ and a differentiable model $f : \bbR^p \times \calX \to \bbR$. If there exists a $\tau$-GD algorithm on $f$, with unbiased initialization $\theta^{(0)} \in \bbR^p$ and $T$ steps such that $\calL_{\calD}(f_{\theta^{(T)}}) < \calL_\calD(0)$, then for $B = \frac{\tau}{\abs{\ell''} C_f}$, there exists $h \in \NTK^{f}_{\theta^{(0)}}(B)$ with $\calL_{\calD}(h) < \calL_\calD(0) - \frac{\tau^2}{2\abs{\ell''} C_f}$.
}
\end{theorem}

\begin{proof}
At a high level, we argue that if gradient descent is going to move at all, the gradient at initialization must be substantially non-zero, and hence move the model in some direction that improves over initialization.  But since the tangent kernel approximates the true model close to the initialization, this improvement translates also to an improvement over initialization also in the Tangent Kernel.  If the initialization is unbiased, initialization is at the null predictor, and this is an improvement over null predictions.

Let us consider first the general (not unbiased) case, and establish \eqref{eq:thmgen}. If $\calL_{\calD}(f_{\theta_0})\leq\eps$, then $f_{\theta_0}\in\NTK^{f}_{\theta_0}(C_f)$ already satisfies \eqref{eq:thmgen}.  Otherwise, we must have that $\norm{\nabla_\theta \calL_\calD(f_{\theta_0})}_2 \ge \tau$, since otherwise $g_t$ can be $0$ at every iteration of gradient-descent, forcing gradient-descent to return $f_{\theta_0}$.
Denote \removed{$\calL(\Bw) := \calL_\calD(h_{\Bw})$ where} $h_\Bw(x) = f_{\theta_0}(x)+\inner{\Bw, \nabla_\theta f_{\theta_0}(x)}$, for $\Bw \in \bbR^p$, noting that $h_{\Bw} \in \NTK^{f}_{\theta_0}((1+\norm{w})C_f)$ and $h_0=f_{\theta_0}$. Observe that:
\begin{align}
\nabla_\theta \calL_\calD(f_{\theta_0})
~=~ \mean{(x,y) \sim \calD}{\ell'(h_0(x), y)\nabla_\theta f_{\theta_0}(x)} ~=~ \nabla_{\Bw} \calL(h_0)
\end{align}

For any $\alpha > 0$, denoting $\Bw_* = -\alpha \frac{\nabla_{\Bw} \calL(h_0)}{\norm{\nabla_{\Bw} \calL(h_0))}}$, we have:
\begin{align*}
\calL(\Bw_*)
&~=~ \Ex_{(x,y) \sim \calD} \ell(\inner{\Bw_*, \nabla_\theta f_{\theta_0}(x)}, y)\\
&~\le~ \Ex_{(x,y) \sim \calD}%
	\inbmat{%
	\ell(h_0(x), y) ~+~
	\ell'(h_0(x),y)\inner{\Bw_*, \nabla_\theta f_{\theta_0}(x)} ~+~
	\frac{\abs{\ell''}}{2} \inner{\Bw_*, \nabla_\theta f_{\theta_0}(x)}^2} \\
&~\le~ \calL(h_0) + \inner{\Bw_*, \nabla_{\Bw} \calL(h_0)} + \frac{\abs{\ell''}}{2}\norm{\Bw_*}^2 C^2_f\\ 
&~=~ \calL(h_0) - \alpha \norm{\nabla_{\Bw} \calL(h_0)} + \frac{\alpha^2\abs{\ell''} C^2_f}{2}\\
&~<~ \calL(h_0) - \alpha \tau + \frac{\alpha^2 \abs{\ell''} C^2_f}{2}.
\end{align*}
Choosing $\alpha = \frac{\tau}{\abs{\ell''} C^2_f}$ completes the proof of \eqref{eq:thmgen}.  For unbiased initialization, if $\eps < \calL_\calD(0)=\calL_\calD(f_{\theta_0})$ then again we must have $\norm{\nabla_\theta \calL_\calD(f_{\theta_0})}_2 \ge \tau$ and the proof proceeds as above, noting that $f_{\theta_0}=h_0=0$.
\end{proof}

\autoref{sep:1} establishes that \autoref{thm:unbiased} is polynomially tight: it shows that despite gradient descent succeeding with an unbiased initialization, the tangent kernel at initialization has an edge of at most $O(\tau/C_f)$.  This leaves a quadratic gap versus the guarantee of $\Omega(\tau^2/C^2_f)$\nati{I changed quartic to quadratic.  Please double check.}\info{PK: looks good to me...} in the Theorem (recalling $\abs{\ell''}=1$ for the square loss), but establishes the polynomial relationship.

\autoref{thm:unbiased} is a strong guarantee, in that it holds for each source distribution separately.  This immediately implies that if $\tau$-approximate gradient descent on a differentiable model $f$ with unbiased initialization $\theta_0$ learns some family of distribution $\calP$ to within small error, or even just to with some edge over the null prediction, then the Tangent Kernel at $\theta_0$ also has an edge of at least $\frac{\tau^2}{2\abs{\ell''} C^2_f}$ over the null prediction.  Note that this edge is polynomial in the algorithm parameters $\tau$ and $C_f$, as is the required norm $B$.  And so, if, for increasing problem sizes $n$, gradient descent allows for ``polynomial learnability'', in the sense of learning with polynomially increasing complexity, then the Tangent Kernel at least allows for ``weak learning'' with a polynomial edge (in all the parameters, and hence in the problem size $n$).  \autoref{sep:2} shows that this relationship is tight, in that there are indeed problems where strong learning is possible with gradient descent with polynomial complexity, but the tanget kernel, or indeed any kernel of polynomial complexity, can only ensure polynomially weak learning.

\autoref{thm:unbiased} relies on the initialization being unbiased, or at the very least the predictions at initialization not being worse than the null predictions.  Can we always ensure the initialization satisfies such a condition?  And what happens if it does not?  Might it then be possible for Gradient Descent to succeed even though the Tangent Kernel does not have an edge over the null prediction?

The example in \autoref{sec:no-weak-learn} shows that, indeed, some learning problems might be learnable by gradient descent, but only using an initilization that is {\em not} unbiased\removed{(i.e.~no reasonably sized differentiable model with an unbiased initialization can ensure small error, but there exists a not-unbiased initialization that does ensure arbitrarily small error)}.  That is, we should not expect initializations to always be unbiased, nor can we expect to always be able to ``fix'' the initialization to be unbiased.  And furthermore, the example shows that with an initilization that is not unbiased, the Tangent Kernel at initilization might not have a significant edge over the null predictor.

But before seeing this example, let us ask: What can we ensure when the initialization is not unbiased?  

\subsection{With Distribution Dependence} 
\label{subsec:dist-dependent}

In \autoref{thm:dist-dependent} we show that, at least for the square loss, for an arbitrary initialization, even though the Tangent Kernel at initialization might not have an edge, we can construct a distribution $\calW$, which we can think of as an alternate ``random initilization'', such that the Tangent Kernel of the model at a random point $\theta\sim\calW$ drawn from this distribution, {\em does} have an edge over null prediction.  But the catch is that this distribution depends not only on the true initialization $\theta_0$, but also on the input distribution $\calD_\calX$.  That is, the Tangent Kernel for which we are establishing an edge is {\em distribution dependent}. In \autoref{sec:no-weak-learn} we will see that this distribution-dependence is unavoidable.

\removed{
We show that weak learning using GD is not stronger than weak learning using a {\em randomized} and  {distribution-dependent} NTK. That is, we consider an NTK that is defined at a randomized initialization instead of a fixed one (different from the initialization for gradient descent) and moreover, where we allow this initialization to depend on the marginal $\calD_{\calX}$.
}

\begin{theorem}\label{thm:dist-dependent}
For a differentiable model $f$, initialization $\theta_0$, stepsize $\eta$, number of iterations $T$, and the square loss, given an input distribution $\calD_\calX$, there exists a distribution $\calW$ (that depends on $\calD_\calX$), supported on $T+1$ parameter vectors in $\bbR^p$, such that for any source distribution $\calD$ that agrees with $\calD_\calX$ and for which $\tau$-approximate gradient descent ensures error $\calL_{\calD}(f_{\theta^{(T)}}) < \calL_\calD(0) - \gamma$, we have
\begin{equation}\label{eq:thm:dist}
    \Ex_{\theta \sim \calW}~\inf_{h \in \NTK^{f}_{\theta}(B)}~\calL_{\calD}(h) < \calL_\calD(0)-\gamma'
\vspace{-3mm}
\end{equation} 
where $\gamma' =  \min \inbrace{\frac{\gamma}{T+1}, \frac{\tau^2}{2 C^2_f}}$ and $B = \sqrt{2 \gamma'}C_f$. 
And so, the Average Tangent Kernel $K=\Ex_{\theta \sim \calW} \NTK_\theta^f$ has rank at most $(T+1)(p+1)$ and
\begin{equation}
    \inf_{h\in\calF(K,B)} \calL_\calD(h)\leq \calL_\calD(0)-\gamma'.
\end{equation}
\end{theorem}

\begin{proof}
\removed{
Suppose for contradiction that 
there exists some $\calD \in \calP$ such that for every initialization $\calW(\calD_\calX)$ we have:
\begin{equation}
\Ex_{\theta \sim \calW(\calD_\calX)} ~ \inf_{h \in \NTK^{h}_{\theta}(B)} ~ \calL_{\calD}(h) ~\ge~ \calL_{\calD}(0) - \gamma' \label{eqn:dist-dep-1}
\end{equation}}
We first define a sequence of iterates $\tilde{\theta}_t$ which corresponds to gradient descent descent on $\calL_{\tilde{\calD}}(f_\theta)$, where $\tilde{\calD}$ is defined as a distribution with marginal $\calD_\calX$ and $y=0$ with probability one.  These iterates are given by:
\begin{align}
    \tilde{\theta}_0 &= \theta_0 \\
    \tilde{\theta}_{t+1} &= \tilde{\theta}_{t} - \eta  \mean{x \sim \calD_\calX}{f_{\tilde{\theta}_{t}}(x) \cdot \nabla_{\theta_{t}} f_{\tilde{\theta}_{t }}(x)} 
\end{align}
Let $\calW$ be the uniform distribution over $\{\tilde{\theta}_0, \dots, \tilde{\theta}_T\}$.  At a high level, if the Tangent Kernel at all of these iterates does not have an edge, then gradient descent will not deviate from them, in the same way that in \autoref{thm:unbiased} the Tangent Kernel had to have an edge in order for gradient descent to move.  But then gradient descent leads to $f_{\tilde{\theta}_T}\in\NTK^f_{\tilde{\theta}_t}$, and so any edge gradient descent has is also obtained in this Tangent Kernel.  

In a sense, we are decomposing the gradient into two components\nati{It would be really nice to write down an equation that shows this decomposition and refer to it more directly.  In fact, it would be good to do even before defining $\tilde{\theta}$, and then have this directly motivate the definition of $\tilde{\theta}$.  We can even define $\calD_0$ which is $\calD_X$ composed with zero labels, and then write it down in terms of the gradient of the loss on $\calD_0$.}: the gradient if all labels were zero, and the deviation from this gradient due to the labels not being zero.  The second component is non-zero only when the tangent kernel has an edge.  If the initialization is unbiased, the first component is zero, and so if the tangent kernel doesn't have an edge, gradient descent will not move at all.  But if the initialization is not unbiased, the first component (corresponding to zero labels) might be non-zero even if the tangent kernel does not have an edge, and so gradient descent can explore additional iterates $\tilde{\theta}_t$, and any one of these having the edge could be sufficient for gradient descent succeeding.  The key is that this sequence of iterates depends only on the input distribution $\calD_\calX$, but not on the labels (since it is defined as the behaviour on zero labels).

Denote $v^{(t)} := \mean{(x,y) \sim \calD}{y \cdot \nabla_\theta f_{\tilde{\theta}^{(t)}}(x)}$, and let $\Bw^{(t)} = -\alpha{v^{(t)}}/{\norm{v^{(t)}}}$, with $\alpha$ to be set later s.t.~$(x\mapsto \inangle{ \Bw^{(t)} , \nabla_\theta f_{\tilde{\theta}^{(t)}}(x)})\in\NTK^f_{\tilde{\theta}^{(t)}}(B)$.  Assume for contradiction that \eqref{eq:thm:dist} does not hold (i.e.~the Tangent Kernels do not have an edge on average), and so:
\begin{align}
&\E_{t\sim[0..T]} ~ \mean{(x,y) \sim \calD}{\ell \left(\inner{\Bw^{(t)}, \nabla_\theta f_{\tilde{\theta}^{(t)}} (x)}, y\right)}
~\ge~ \E_{\theta \sim \calW}\inf_{h \in \NTK^{f}_{\theta}(B)} \calL_{\calD}(h) ~\ge~  \calL_{\calD}(0) - \gamma'
\end{align}
On the other hand, we have for every $t$:
\begin{align*}
&{\mean{(x,y) \sim \calD}{\ell\left(\inner{\Bw^{(t)},  \nabla_\theta f_{\tilde{\theta}^{(t)}} (x)} ,y\right)}}\\
&=~ {\Ex_{(x,y) \sim \calD}
    \insquare{\frac12 y^2
    ~-~ y \inner{\Bw^{(t)}, \nabla_\theta f_{\tilde{\theta}^{(t)}} (x)}
    ~+~ \frac{1}{2} \inner{\Bw^{(t)}, \nabla_\theta f_{\tilde{\theta}^{(t)}} (x)}^2}} \\
&\le~ \E_{(x,y) \sim \calD} \frac12 y^2 - \inner{\Bw^{(t)}, \E_{(x,y) \sim \calD} y \nabla_\theta f_{\tilde{\theta}^{(t)}} (x)}
~+~ \frac{1}{2} \E_{(x,y) \sim \calD} \|\Bw^{(t)}\|^2 \norm{\nabla_\theta f_{\tilde{\theta}^{(t)}} (x)}^2 \\
&\le~ \calL_{\calD}(0) - \alpha \norm{v^{(t)}}  + \frac{\alpha^2}{2} C^2_f
\end{align*}
Using the above, taking $\alpha = \frac{\sqrt{2\gamma'}}{C_f}$ we get that for every $t$:
\begin{equation}
\label{eq:grad_bound}
\|v^{(t)}\|
\le (T+1)
\E_{t \sim [0..T]}\|v^{(t)}\| \le (T+1) \sqrt{2 \gamma'}C_f
\end{equation}
We will now argue inductively that $\tilde{\theta}^{(t)}$ is a valid choice for the iterate $\theta^{(t)}$, i.e.~satisfying \eqref{eq:gd}.  For iteration $t=0$, $\tilde{\theta}^{(0)}=\theta^{(0)}=\theta_0$. If up to iteration $t$ we can have $\tilde{\theta}^{(t)}=\theta^{(t)}$ then:
\begin{align}
\nabla_{\theta} \insquare{\E_{(x,y) \sim \calD} \ell(f_{\theta^{(t)}}(x), y)} 
&~=~ \E_{(x,y) \sim \calD}\  (f_{\theta^{(t)}}(x) - y) \cdot \nabla_{\theta} f_{\theta^{(t)}}(x) \\
&=~ \E_{(x,y) \sim \calD}\  (f_{\tilde{\theta}^{(t)}}(x) - y) \cdot \nabla_{\theta} f_{\tilde{\theta}^{(t)}}(x) \\
&~=~ \E_{x \sim \calD_\calX}\  f_{\tilde{\theta}^{(t)}}(x) \cdot \nabla_{\theta} f_{\tilde{\theta}^{(t)}}(x) -  v^{(t)}
\end{align}
Now, from \eqref{eq:grad_bound}, and since $\tau \ge (T+1) \sqrt{2 \gamma'}C_f$,  we have that $g_t=\E_{x \sim \calD_\calX}f_{\theta^{(t)}}(x) \nabla_{\theta} f_{\theta^{(t)}}(x)$ satisfies \eqref{eq:gt}, and so $\theta^{(t+1)} = \tilde{\theta}^{(t+1)}$ satisfies \eqref{eq:gd}.  And by induction, GD can return $f_{\theta^{(T)}} = f_{\tilde{\theta}^{(T)}}$.
Since $f_{\tilde{\theta}^{(T)}} \in \NTK^{f}_{\tilde{\theta}^{(T)}}(B)$ we get:
\begin{align*}
\calL_{\calD}(0) - \calL_{\calD}(f_{\theta^{(T)}})
&~=~ \calL_{\calD}(0) - \calL_{\calD}(f_{\tilde{\theta}^{(T)}})\\
&~\le~ \calL_{\calD}(0) - \inf_{h \in \NTK^{f}_{\tilde{\theta}^{(T)}}(B)}\calL_{\calD}(h) \\
&~\le~ (T+1) \mean{\theta \sim \calW(\calD_\calX)}{\calL_{\calD}(0) - \inf_{h \in \NTK^{f}_{\theta}(B)}\calL_{\calD}(h)}\\
&~\le~ \gamma' (T+1) ~\le~ \gamma\,.
\end{align*}
This leads to a contradiction, thereby completing the proof of \autoref{thm:dist-dependent}.  For the average Kernel, note that it corresponds to a $(T+1)(p+1)$ dimensional feature space which is the concatenation of the feature spaces of each $\NTK^f_{\tilde{\theta}^{(t)}}$, and consider a predictor which is the concatenation of the best predictors for each kernel scaled by $1/(T+1)$.
\end{proof}

\section{No Edge with the Tangent Kernel}\label{sec:no-weak-learn}

As promised, we will now present an example establishing:
\begin{enumerate}[noitemsep,topsep=0pt,parsep=0pt,partopsep=0pt,leftmargin=*]
    \item Some learning problems are learnable by gradient descent, but only with a biased initialization
    \item \autoref{thm:unbiased} cannot be extended to biased initializations: Even if gradient descent, with a {\em biased} initialization, ensures small error, the Tangent Kernel at initialization might not have a significant edge.
    \item \autoref{thm:dist-dependent} cannot be made distribution-independent: Even if gradient descent, with a biased initialization, ensures small error on a learning problem, there might not be any {\em distribution-independent} random ``initialization'' that yields a significant edge in a Tangent Kernel at a random draw from this ``initialization''.
\end{enumerate}



\paragraph{The Learning Problem.} This time we will consider learning parities of any cardinality (not necessarily sparse), but we will ``leak'' the support of the parity through the input distribution $\calD_\calX$. For an integer $n$ and $\alpha \in (0,1)$, we consider the ``leaky parities'' problem $\Plp[n,\alpha]$ over $\calX = \sbit^n$ and $\calY = \sbit$ with $2^n - n - 1$ distributions, each corresponding to $I \subseteq [n]$ with $|I| \ge 2$.  Each source distribution $\calD_I$ is a mixture of two components $\calD_I=(1-\alpha)\calD_I^{(0)}+\alpha \calD_I^{(1)}$: the first component $\calD_I^{(0)}$ has labels $y=\chi_I(x)$ corresponding to parity of bits in $I$, with a uniform marginal distribution over $x$.  The second component has uniform random labels $y$ independent of the inputs (i.e.~the labels in this component are entirely uninformative), but the input is deterministic and set to $x := x^I$ where $x^I_i = 1$ for $i \in I$ and $-1$ for $i \notin I$.  Learning this problem is very easy: all a learner has to do is examine the marginals of each input coordinate of $x_i$.  Those in the support $I$ will have $\Ex_{\calD_I}[x_i]=\alpha$ while when $i\not\in I$ we have $\Ex_{\calD_I}[x_i]=-\alpha$.  The marginals over $x_i$ thus completely leak the optimal predictor.  But as we shall see, while gradient descent can leverage the marginals in this way, kernel methods (where the kernel is pre-determined and does not depend on the input distribution) fundamentally cannot.

More formally, $\calD_I^{(0)}$ is the distribution obtained by sampling $x \sim \calU(\{\pm 1\}^n)$ and setting $y = \chi_I(x) := \prod_{i \in I} x_i$, while $\calD_I^{(1)}$ is the distribution obtained by (deterministically) setting $x := x^I$ and sampling $y \sim \calU(\sbit)$, and recall $\calD_I := (1-\alpha) \calD_I^{(0)} + \alpha \calD_{I}^{(1)}$. In other words, $\calD_I$ corresponds to first choosing $b \in \bit$ with $\Pr[b=1] = \alpha$ and sampling $(x,y) \sim \calD_I^{(b)}$.

\paragraph{The differentiable model.} To learn $\Plp[n,\alpha]$, we construct a differentiable model that is very similar to the one in \autoref{sec:separation}, being a combination of a linear predictor and a ``selected-parity'', but where the linear component has a bias term, and is thus equal to $-1$ (rather then zero) at initialization.  Formally, we construct $f_{\theta}$ with $\theta \in \bbR^n$ (i.e. size $p=n$) that behaves as follows
\begin{equation}\label{eq:f-theta-sec5}
\forall \theta \in \inparen{\textstyle \insquare{0,\frac{2}{n}} \cup \insquare{\frac{3}{n},\frac{5}{n}}}^n \quad : \quad
f_{\theta}(x)
\begin{cases}
	\approx -1 + 2 \inangle{\theta,x + \frac53\alpha\mathbf{1}} & \text{if } \theta \approx 0\\
	= \prod_{i\,:\,\theta_i \ge \frac3n} x_i & \text{if } \exists i \,:\, \theta_i \ge \frac{3}{n}
\end{cases},
\end{equation}
where $\approx$ stands for the first order approximation of $f_{\theta}$ about $\theta=0$, and $\mathbf{1} \in \bbR^n$ is the all-$1$s vector.  Again, \eqref{eq:f-theta-sec5} is the {\em only} property we need, to show that approximate gradient descent learns $\Plp[n,\alpha]$: at initialization, we output $-1$ for all examples, meaning the derivative of the loss is non-negative, and we generally want to try to increase the output of the model.  How much we increase each coordinate $\theta_i$ will depend on $\Ex[x_i+\frac{5}{3}\alpha]$, which is $8/3$ for $i \in I$ but only $2/3$ for $i\not\in I$, and the first step of gradient descent will separate between the coordinates inside and outside the support $I$, thus effectively identifying $I$.  With an appropriate stepsize, the step will push $\theta$ into the regime covered by the second option in \eqref{eq:f-theta-sec5}, and the model will output the desired parity\footnote{The reason for the offset $\frac{5}{3}$ is to ensure all coordinates will be non-negative after the first step, which simplifies identification of the two regimes in \eqref{eq:f-theta-sec5} when implementing the model}.  The key ingredient here is that even though all the coordinates are uncorrelated with the labels, and thus also with residuals at an {\em unbiased} initialization, and the model of \autoref{sec:separation} will have zero gradient at initialization, we artificially make all the residuals at initialization non-positive, thus creating a correlation between each coordinate and the residual, which moves $\theta$ in a way that depends on the marginal $\calD_X$ (as in the proof of \Cref{thm:dist-dependent}).

As with the model of \autoref{sec:separation}, we can implement \eqref{eq:f-theta-sec5} as a continuous differentiable model with scale $C_f= O(n)$, through a slight modification of the architecture described in Equations \eqref{eq:def-f-1}-\eqref{eq:def-f-4}:
\begin{align}
f_\theta(x) &~:=~
\sigma_{0,\frac{2}{n}}^{-1,0}\inparen{\inangle{\theta,\mathbf{1}}} + \sigma_{-1,1}^{-1,1} \left( \inner{\theta, x + {\textstyle \frac53} \alpha \mathbf{1}} + \calG(\theta, x) \right) \label{eq:def-f5-1}\\
\calG(\theta, x) &~:=~ S(\xi(\theta \circ x)) \cdot \inparen{H(\xi(\theta \circ x)) - \inangle{\theta, x + {\textstyle \frac53} \alpha \mathbf{1}}} \label{eq:def-f5-2}
\end{align}
where $S$, $H$ and $\xi$ are as defined in \eqref{eq:def-f-3}, \eqref{eq:def-f-3.5} and \eqref{eq:def-f-4}. The first term in \eqref{eq:def-f5-1} is $-1$ at $\theta^{(0)} = 0$ and satisfies $\nabla_{\theta}~\sigma_{0,\frac{2}{n}}^{-1,0}\inparen{\inangle{\theta^{(0)},\mathbf{1}}} = 0$. The second term is essentially the same as \eqref{eq:def-f-1}, with only difference being the linear term $\inangle{\theta, x}$ is replaced with $\inangle{\theta, x + \frac53\alpha \mathbf{1}}$. Similar to the construction in \autoref{sec:separation}, we get that \eqref{eq:f-theta-sec5} holds and that $C_f^2 \le O(n^2)$.  Furthermore, we can implement $f_\theta$ using a similar feed-forward network of depth $O(\log n)$ and $O(n)$ edges.\info{There was a footnote here about removing weight sharing; I have commented it out because I haven't fully worked it out. See \LaTeX\xspace source.}

We prove the following claims in \Cref{apx:neural-net-ub-2} and \Cref{apx:tang-kernel-lb-2} respectively (cf. \href{clm:gd-upper-bound-1}{Claims \ref{clm:gd-upper-bound-1}} and \ref{clm:ntk-edge-1}).

\begin{restatable}{claim}{gdubtwo}\label{clm:gd-upper-bound-2}
For any $n$ and $\alpha \in (0,1)$, for any $\calD_I \in \Plp[n,\alpha]$, gradient descent on the model $f$ of size $p=n$ with $C^2_f=O(n^2)$ described above, with initialization $\theta_0 = 0$, accuracy $\tau = \frac43\alpha$, step size $\eta=1$ and $T=1$ step ensures error $\calL_{\calD_I}(f_{\theta^{(T)}}) \le \alpha$.
\end{restatable}

\begin{restatable}{claim}{ntkedgetwo}\label{clm:ntk-edge-2}
The tangent kernel of the model $f$ at $\theta_0=0$ is the scaled linear kernel endowed with a bias term, $\NTK_{\theta_0}^f(x,x')=(1 + \frac{100}{9}\alpha^2) + 4 \inangle{x,x'}$, and for any $n\geq 2$, $\alpha\in(0,1)$, and any $\calD_I\in\Plp[n,\alpha]$, the error with this kernel is $\inf_{h\in\NTK_{\theta_0}^f(B)}\calL_{\calD_I}(h)=\calL_{\calD_I}(0)=\frac{1}{2}$ for any $B \ge 0$.
\end{restatable}

We already see that in contrast to a situation where the initialization is unbiased, and so \autoref{thm:unbiased} ensures the tangent kernel has at least a polynomial edge, even if gradient descent succeeds, but with an initialization which is not unbiased, the tangent kernel at initialization might not have any edge at all, and so we cannot hope to remove the requirement of the initialization being unbiased from \autoref{thm:unbiased}:

\begin{sepres}\label{sep:3}
For any $\eps>0$, there exists a source distribution (with binary labels and w.r.t.~the square loss), such that
\begin{itemize}[leftmargin=*]
\item {[Gradient Descent]} Using a differentiable model with $p=2$ parameters, $T=1$ steps, and accuracy $\frac{\tau}{C_f}=\Theta(\eps)$, $\tau$-approximate gradient descent ensures square loss of $\eps$, but
\item {[Tangent Kernel]} The tangent kernel of the model at initialization cannot ensure square loss lower than $\frac{1}{2}$ (which is the loss of null prediction).
\end{itemize}
\end{sepres}
\begin{proof}
Apply Claims \ref{clm:gd-upper-bound-2} and \ref{clm:ntk-edge-2} to the sole source distribution in $\Plp[n=2,\alpha=2\eps]$.
\end{proof}

\noindent But what about the tangent kernel at a different ``initialization'', or of some other model, or even some other kernel altogether?  The following Claim (proved in \Cref{apx:kernel-lb-3}) shows that no kernel can get a significant edge over the zero predictor unless the number of features and the norm are both exponentially large in $n$.  In order to contrast with guarantee \eqref{eq:thm:dist} of \autoref{thm:dist-dependent}, we state the Claim also for learning using a random kernel, i.e.~in expectation over an arbitrary distribution of kernels.
\begin{restatable}{claim}{kerlbtwo}\label{clm:ker-lower-bound-2}
For all $\alpha \in (0,1), p, B, n$, and any randomized kernel $K$ with $\textrm{rank}(K)=p$ almost surely~(i.e.~a distribution over kernels, each of which corresponding to a $p$-dimensional feature map), there exists $\calD_I \in \Plp[n,\alpha]$ such that
\begin{align*}
\Ex_{K} ~ \inf_{h \in \calF(K,B)} ~ \calL_{\calD_I}(h)
~\ge~ \calL_{\calD_I}(0)- \min\set{\frac{p}{2|\Plp|}, O\inparen{\frac{B^{\nicefrac{2}{3}}}{|\Plp|^{\nicefrac{1}{3}}}}}\,.
\end{align*}
where $\calL_{\calD_I}$ is w.r.t.~the square loss, and note that $|\Plp[n,\alpha]| = 2^n - n - 1$.
\end{restatable}

In particular, we see that in contrast to the distribution-dependent situation of \autoref{thm:dist-dependent}, where we could ensure a polynomial edge for the tangent kernel at a specified randomized ``initialization'', this is not possible in general, and we cannot hope to remove the dependence on the input distribution in \autoref{thm:dist-dependent}.

\begin{sepres}\label{sep:4}
For any sequence $\eps_n = 1/\poly(n)$, there exists a sequence of learning problem $\calP_n$ (with binary labels and w.r.t.~the square loss), such that
\begin{itemize}[leftmargin=*]
\item {[Gradient Descent]} for each $n$, using a differentiable model with $p=n$ parameters, realizable by a neural network of depth $O(\log n)$ and $O(n)$ edges (where some edges have trainable weights and some do not), $T=1$ steps, and accuracy $\frac{\tau}{C_f}=O(\frac{\eps_n}{n})$, $\tau$-approximate gradient descent learns $\calP_n$ to square loss of $\eps_n$,\info{there was a phrase here about this being Bayes-optimal. But that is not exactly true! We are $2^{-\Theta(n)}$ close to Bayes-optimal, but not exactly.} but
\item {[Poly-Sized Kernels]} no sequence of (randomized) kernels $K_n$ corresponding to feature maps of dimension $\poly(n)$ (and hence no tangent kernel of a poly-sized differentiable models, even with randomized ``initialization'') can allow learning $\calP_n$ (in expectation) to square loss smaller than $\frac{1}{2}-2^{-\Omega(n)}$ for all $n$; and 
\item{[Arbitrary Kernel, Poly Norm]} no sequence of (randomized) kernels $K_n$ of any dimension can allow learning $\calP_n$ (in expectation) to square loss smaller than $\frac{1}{2}-2^{-\Omega(n)}$ for all $n$ using predictors in $\calF(K_n,B_n)$ of norm $B_n = \poly(n)$.
\end{itemize}
\end{sepres}
\vspace{-5mm}\begin{proof}
Apply Claims \ref{clm:gd-upper-bound-2} and \ref{clm:ker-lower-bound-2} to $\calP_n=\Plp[n,\alpha=\eps_n]$.
\end{proof}

\noindent 
Since we are working over a domain $\calX=\{-1,+1\}^n$ or cardinality $2^n$, we can always ensure an edge of $(1-\alpha)p/2^{n+1}$ using indicator features on $p$ of the $2^n$ inputs, thus allowing memorization of these tiny fraction of inputs.
An edge of $2^{-\Theta(n)}$ is thus a ``trivial'' edge attained by this kind of memorization, and \autoref{clm:ker-lower-bound-2} and \autoref{sep:4} establish the tangent kernel, or even any other kernel, cannot be significantly better.

We can also conclude that even though we saw that our learning problem is learnable with gradient descent, we cannot hope to learn it with an unbiased initialization, since this would imply existence of a kernel with a polynomial edge.  We thus get the following separation between learnability with gradient descent using biased versus unbiased initialization (importantly, only for distribution {\em independent} learning).

\begin{sepres}\label{sep:5}
For any sequence $\eps_n = 1/\poly(n)$, there exists a sequence of learning problem $\calP_n$ (with binary labels and w.r.t.~the square loss), such that
\begin{itemize}[leftmargin=*]
\item {[GD with \textbf{biased} initialization]} for each $n$, using a differentiable model with $p=n$ parameters, realizable by a neural network of depth $O(\log n)$ and $O(n)$ edges (where some edges have trainable weights and some do not), $T=1$ step, and accuracy $\frac{\tau}{C_f}=O\left(\frac{\eps_n}{n}\right)$, and some (\textbf{not unbiased}) initialization, $\tau$-approximate gradient descent learns $\calP_n$ to square loss $\eps_n$, but 
\item {[GD with \textbf{unbiased} initialization]} It is not possible to ensure error better than $\frac{1}{2}$ on $\calP_n$, for all $n$, using $\tau$-approximate gradient descent starting with an \textbf{unbiased} initialization (as in Eq.~\eqref{eq:unbiased}), and any differentiable models of size $p = \poly(n)$ and accuracy $\tau/C_f = 1/\poly(n)$, and any number of steps $T$.
\end{itemize}
\end{sepres}
\begin{proof}
Consider $\calP_n=\Plp[n,\alpha =\eps_n]$. \autoref{clm:gd-upper-bound-2} ensures learnability with a biased initialization. Suppose $\calP_n$ was also learnable by $\tau$-approximate gradient descent with unbiased initialization, using a model $f$ of size $p_n = \poly(n)$, scale $C_{f}$ satisfying $\tau/C_{f} = 1/\poly(n)$, to within any error better than $0.5$, then \autoref{thm:unbiased} would imply the tangent kernel at initialization would ensure error $\le \frac12 - \frac{1}{\poly(n)}$. This is a fixed kernel that depends only on the model $f_{\theta}$, and corresponds to a feature map of dimension $p_n+1$. Since $p_n = \poly(n)$, this would contradict \autoref{clm:ker-lower-bound-2}.
\end{proof}


\renewcommand{\ssl}[1]{\makecell[l]{#1}}
\renewcommand{\ssc}[1]{\makecell[c]{#1}}

\newcommand{\unbiased}{$0$} 
\newcommand{\biased}{\texttimes} 
\newcommand{\apxunbiased}{$\!\approx$} 
\newcommand{\perinst}{--\xspace} 
\newcommand{\distdep}{$y$\xspace} 
\newcommand{\distindep}{${\bm x}$\xspace} 

\newcommand{\gdErrCol}{\cellcolor{OliveGreen!10}}
\newcommand{\kerErrCol}{\cellcolor{Gred!10}}

\newcommand{\gdErrBold}[1]{\textcolor{black!20!OliveGreen}{\boldmath #1}}
\newcommand{\kerErrBold}[1]{\textcolor{Gred}{\boldmath #1}}

\begin{table}[!ht]
\caption{\label{tab:prior-sep}
Prior and new separation results between differentiable learning and kernel methods}
\vspace{-3mm}
\footnotesize
\renewcommand{\arraystretch}{1.7}
\setlength\tabcolsep{4pt}
\begin{center}
\begin{tabular}{!{\vrule width 1.1pt} l|l !{\vrule width 1.1pt} l|l|l|p{2mm} !{\vrule width 1.1pt} l|l|p{2mm} !{\vrule width 1.1pt}}
\noalign{\hrule height 1.1pt}
\MC{2}{!{\vrule width 1.1pt}c!{\vrule width 1.1pt}}{\cellcolor{myGold!50}\bf Learning Problem} &
\MC{4}{c!{\vrule width 1.1pt}}{\cellcolor{Ggreen!50}\bf Learning with Gradient Descent} &
\MC{3}{c!{\vrule width 1.1pt}}{\cellcolor{Gred!50}\bf Kernel Lower Bound}\\
\cline{1-9}
{\bf Labels} &
{\bf Inputs} &
{\bf Model} &
{\bf Type of GD} &
{\bf \gdErrBold{Error}} &
{\bf \,I} & 
{\bf Kernel} &
{\bf \kerErrBold{Error}} &
{\boldmath$\!\calD$}\\ 
\noalign{\hrule height 1.1pt}
& &
& & & 
\MR{5}{\unbiased} &
\MC{2}{c|}{\ssc{\cite{yehudai19power}}} & \\
\MR{4}{\ssl{Single ReLU\\ $y=[\inangle{w,x}+b]_+$}} &
\MR{4}{\ssl{Gaussian}} &
\MR{4}{\ssl{Single ReLU\\ Zero Init}} &
\MR{4}{\ssl{\\[1mm] Batch GD\\[0.5mm] \scriptsize $\poly(n \log\frac{1}{\eps})$ \\\scriptsize steps\\[0.2mm] \scriptsize $\poly(\frac{n}{\eps})$ samples}} &
\MR{4}{\ssl{\gdErrBold{$\eps\rightarrow 0$} \\[3mm] when \\[-.5mm] $b=0$}} &
 & dim $\le 2^{O(n)}$ & 
\ssl{\kerErrBold{$\Omega(1/n^8)$ : \scriptsize \bf see \ref{enum:ys}}}& 
\distdep\\
\cline{7-9}
& & & & & &
\MC{2}{c|}{\cite{ghorbani19linearized}} & \\
& & & & & &
\ssl{{\bf rotationally inv.}\\ \#samples $\le n^c$}& 
\MR{2}{\ssl{\\[1mm]\kerErrBold{$\geq \eps'(c)$} $>0$\\[3mm]\scriptsize even if $b=0$ }}&
\MR{2}{\perinst}\\
\cline{7-7}
& & & & & &
\ssl{\\[-2.8mm]{\bf NTK} of depth-$2$\\ ReLU Net\\ with $\leq n^c$ units }&
& \\
\cline{1-9}
Quadratic function &
Gaussian &
\MR{2}{\ssl{\\[-2.9mm]Depth-$2$ Net\\ quad.~activations\\ \# units = $\rho \cdot n$\\ for $\rho < 1$}} &
\MR{2}{\ssl{\\Gradient Flow\\[.5mm] on Population}} &
\MR{2}{\ssl{\\\gdErrBold{${\rightarrow\eps(\rho)}$}\\[1mm] \scriptsize see \ref{enum:rho}}} &
\MR{2}{\apxunbiased} &
\MC{2}{c|}{\cite{ghorbani19limitations}} &
\MR{2}{\perinst} \\
\cline{1-2}
\MC{2}{!{\vrule width 1.1pt}l!{\vrule width 1.1pt}}{{\ssl{\\[-2.2mm]Binary classification between two\\zero mean Gaussians (diff.~cov.)\\[-2.5mm]\mbox{}}}} 
& & & & & 
{{\bf NTK} at Init} &
\ssl{\kerErrBold{$\to \eps'(\rho)$} $> \eps(\rho)$\\ \scriptsize see \ref{enum:rho}} &  \\
\cline{1-9}
\MR{2}{\ssl{Depth-$2$ Net with \\  restricted weights\\ abs. value activations}} &
\MR{2}{Gaussian} &
\MR{2}{\ssl{Depth-$2$ ReLU\\ $\poly(n)$ units \\ Random Init}} &
\MR{2}{\ssl{Truncated$^{\text{\ref{enum:li-et-al}}}$ GD\\ \scriptsize $\poly(n)$ steps \\ \scriptsize $\poly(n)$ samples}} &
\MR{2}{\ssl{\\[-2.5mm]\gdErrBold{$< \frac{1}{n^{1+\delta}}$}\\[1mm] {\scriptsize for some}\\[-.6mm]{\scriptsize $\delta>0$}}} & 
\MR{2}{\ssl{\biased}} &
\MC{2}{c|}{\cite{li20beyondntk}} & \\
& & & & & &
dim $\le \poly(n)$ &
{\kerErrBold{$\Omega(\frac1n)$} $>\Omega(\eps^{1-\frac{\delta}{2}})$}& 
\distdep\\
\cline{1-9}
\MR{2}{\ssl{$y\in \bbR^{\log n}$ \\ $y= x_S + \eps^{\frac{1}{4}} \chi_S(x) \mathbf{1}$}} & 
\MR{2}{Uniform} &
\MR{2}{\ssl{Depth-$2$ ResNet\\ $\poly(n/\eps)$ units \\ Random Init}} &
\MR{2}{\ssl{SGD\\ $\poly(n/\eps)$ steps}} &
\MR{2}{\ssl{\gdErrBold{$<\eps$}}} &
\MR{2}{\biased} &
\MC{2}{c|}{\cite{allenzhu19resnets}} &
\\
& & & & & &
dim $\le \poly(n)$ &
{\kerErrBold{$> \sqrt{\eps}$}}&
\distdep \\
\cline{1-9}
\MR{2}{\ssl{ResNet-type model\\of depth $\Theta(\log\!\log n)$ \\ with restricted weights}} &
\MR{2}{Gaussian} &
\MR{2}{\ssl{Same as teacher\\{$\poly(n)$ units}  \\ Zero Init}} &
\MR{2}{\ssl{Regularized$^{\text{\ref{enum:az-li-reg}}}$\\ SGD\\\scriptsize $\poly(n/\eps)$ steps}} &
\MR{2}{\gdErrBold{$\eps \to 0$}} &
\MR{2}{\unbiased} &
\MC{2}{c|}{\cite{allenzhu20backward}} &
\\
& & & & & &
dim $\le \poly(n)$ &
{\kerErrBold{$\Omega(\frac{1}{n^{0.01}})$}}&
\distdep\\
\cline{1-9}
\MC{2}{!{\vrule width 1.1pt}l!{\vrule width 1.1pt}}{\MR{2}{\ssl{\\[-1.5mm] Sparse parities with input\\distribution leaking the support}}} &
\MR{2}{\ssl{\\[-2.5mm]Depth-$2$ ReLU6 \\ {\scriptsize $\poly(n)$ units} \\ \scriptsize Unbiased$^{\text{\ref{enum:daniely}}}$ Init}}&
\MR{2}{\ssl{\\[-2.5mm]Regularized \\ Population GD\\\scriptsize $\poly(n)$ steps}} &
\MR{2}{\ssl{\\[-2.3mm]\gdErrBold{$\frac{1}{\poly(n)}$}\\[2mm]\scriptsize hinge loss}} &
\MR{2}{\unbiased} &
\MC{2}{c|}{\cite{daniely20parities}} &
\\
\MC{1}{!{\vrule width 1.1pt}c}{}
& & & & & &
\ssl{dim $\le \poly(n)$\\ norm $\le \poly(n)$}&
\ssl{\kerErrBold{$\Omega(1)$}\\ {\scriptsize hinge loss}} &
\distindep \\
\noalign{\hrule height 1.5pt}
\MC{9}{!{\vrule width 1.1pt}c!{\vrule width 1.1pt}}{\cellcolor{Gblue!20}\textbf{Our Results : Separations \ref{sep:1}, \ref{sep:2}, \ref{sep:3}, \ref{sep:4}}}\\
\hline
$\mathsf{XOR}$ of two bits &
\MR{2}{\ssl{Mix of \\ uniform\\ {\tiny \&} biased \\ bits}} &
\MR{2}{Specialized} &
\MR{4}{\ssl{$\tau$-approximate\\Grad Descent\\{\scriptsize (e.g. GD with}\\{\scriptsize $\Omega(1/\tau^2)$ samples)} \\[2mm] $1$ step}} &
\MR{2}{\gdErrBold{$0$}}&
\MR{2}{\unbiased} &
{\bf NTK} at Init &
\ssl{\\[-2mm]\kerErrBold{$\ge \calL(0) - \tau$}\\ {\scriptsize arbitrarily small $\tau$}}  &
\perinst \\
\cline{1-1}\cline{7-9}
\ssl{$\log n$-sparse parities} &
& & & & &
\ssl{dim $\le \poly(n)$}&
\kerErrBold{$\calL(0) - n^{-\Omega(1)}$} &
\distdep\\
\cline{1-3}\cline{5-9}
\MC{2}{!{\vrule width 1.1pt}l!{\vrule width 1.1pt}}{\MR{2}{\ssl{Mix of parities and input\\distribution leaking the support\\ {\footnotesize (mixture ratio depends on $\eps$)}}}} &
\MR{2}{Specialized} &
&
\MR{2}{\gdErrBold{$\le \eps$}} &
\MR{2}{\ssl{\biased}} &
{\bf NTK} at Init &
\kerErrBold{$\calL(0)$} &
\perinst\\
\cline{7-9}
\MC{2}{!{\vrule width 1.1pt}c!{\vrule width 1.1pt}}{} 
& & & & &
dim $\le \poly(n)$ &
\kerErrBold{$\calL(0) - 2^{-\Omega(n)}$} &
\distindep \\
\noalign{\hrule height 1.1pt}
\end{tabular}
\end{center}
\footnotesize
\vspace{-3mm}
See \autoref{sec:survey} for table description and notation, and \autoref{apx:prior-separations} for further details. \newline
Column \textbf{I}: $0$ indicates unbiased initialization, $\approx$ almost unbiased, $\times$ not unbiased. \newline
Column {\boldmath $\mathbf{\calD}$}: ``\perinst'' indicates lower bound for specific source distribution, ``\distdep'' for learning problem with fixed input distribution $\calD_\calX$, ``\distindep'' for learning problem with variable input distributions.
\fontsize{9pt}{3.5pt}\selectfont
\begin{enumerate}[leftmargin=*,itemsep=0pt,label={(\arabic*)}]
\item\label{enum:ys} There is a discrepancy between the lower bound of \citet{yehudai19power} and the upper bounds on learning a Single ReLU with gradient descent \citep[e.g.][]{soltanolkotabi17learning,mei2016landscape} in that the upper bounds are only without a bias (i.e.~when $b=0$) but the lower bound relies crucially on the bias $b$.  \citeauthor{yehudai19power} thus do not establish a separation between differentiable learning and kernel methods.  See \autoref{apx:prior-separations} for discussion on possible remedies.
\item\label{enum:rho} The gradient descent error $\eps(\rho)$ and kernel error $\eps'(\rho)$ depend on $\rho$ and on the specific source distribution (i.e.~the quadratic target or the covariances of the classes), or rather sequence of source distributions as $n\rightarrow\infty$.  But \citet{ghorbani19limitations} show that for non-degenerate distribution sequences, and any $\rho<1$, the limit $\eps(\rho)$ of the gradient descent error (as $n\rightarrow\infty$) will be strictly lower than a lower bound $\eps'(\rho)$ on the error of the NTK.  
\item\label{enum:li-et-al}\citet{li20beyondntk} analyze a form of GD where large coordinates of the gradient are set to zeros.\removed{If the the gradient is set to zero, then it will just stop.  Is this the case here??}  
\item\label{enum:az-li-reg} \citet{allenzhu20backward} use non-standard but simple regularization.  
\item\label{enum:daniely} \citet{daniely20parities} use random initialization, but units are paired with complimentary initialization, as suggested by \citet{chizat19lazy}, so as to ensure an unbiased initialization. 
\end{enumerate}
\end{table}


\section{Survey of Separation Results} \label{sec:survey}

In Sections \ref{sec:separation} and \ref{sec:no-weak-learn} we described learning problems for which gradient descent learning succeeds, but where the NTK, or indeed any other (reasonably sized) kernel, has either a very small edge, or even no edge at all.  Our emphasis was on establishing not only that the NTK or kernel methods get worse error than gradient descent (as in previous separation results), but on bounding how close to the null prediction they must be (i.e.~how small an edge they must have).  Here we survey such prior separation results, understanding the relationships between them, and how they relate to the new separations from this paper.

\paragraph{Survey Table.} \Cref{tab:prior-sep} summarizes prior results showing separations between the error that can be ensured using gradient descent versus using kernel methods, as well as our Separations \ref{sep:1}, \ref{sep:2}, \ref{sep:3} and \ref{sep:4}.  All problems are over $\calX=\mathbb{R}^n$ or $\{\pm 1\}^n$\removed{, generally with $\calD_X$ Gaussian or uniform over $\{\pm1\}^n$ (unless indicated otherwise),} with scalar labels ($\calY=\mathbb{R}$ or $\{\pm 1\}$) and w.r.t.~the squared loss, except \citet{allenzhu19resnets} which uses multiple outputs ($y$ is a vector), and \citet{daniely20parities} which is w.r.t.~the hinge loss.  

The middle set of green columns indicate what error can be ensured by running the indicated type of gradient descent variant on the indicated type of model.  An error of ``$\rightarrow 0$'' indicates that the learning problem only depends on $n$, and that for every $n$, gradient descent on a fixed model (that depends on $n$ but not $\eps$) can make the error $\eps$ arbitrarily small using $\poly(n/\eps)$ samples and/or steps.
An error of ``$\leq \eps$'' indicates that the learning problem also depends on $\eps$: for all $n$ and $\eps>0$, there exists a learning problem on which gradient descent ensures error $\leq \eps$ (but kernel methods incur larger error). The annotations $0/\approx/\times$ under column ``I'' indicates: {\boldmath$0$} if the initialization used is unbiased, or could be easily made unbiased; {\boldmath$\approx$} if it is nearly unbiased (error at initialization close to null prediction); or {\boldmath$\times$} if the initialization is far from null.

The right set of pink columns shows the lower bound on the error for the kernel, or class of kernels indicated, or under other restrictions (``dim'' is the number of features or rank of the kernel, ``norm $\leq B$'' means restricting to $\calF(K,B)$ as in \eqref{eq:FKB}, NTK is the tangent kernel of the model used for gradient descent unless otherwise specified).  In all cases, the error is normalized so that the error of  null perdition is $\calL(0)=\nicefrac{1}{2}$.  For our separations, the error is given in terms of the edge (if any) over null prediction.  In all prior separations, the error lower bound is bounded away from the null prediction (the edge is at least $0.1$), and the actual error lower bound is indicated.  The notation $\geq \eps'(c)>0$ and $\geq\eps'(\rho)>\eps(\rho)$ indicates that for any choice of $c$ or $\rho$, there is some $\eps'(c)>0$ or $\eps'(\rho)>\eps(\rho)$ which lower bounds the error.  

The last column indicates the nature of the learning problem used to establish the lower bound, and thus whether the separation is distribution dependent or independent (in the sense introduced in \autoref{sec:prelims}): ``{\boldmath$-$}'' indicates that the separation is for a specific source distribution (or sequence of source distribution with increasing input dimension $n$).  This is the strongest form of separation, implying also distribution dependent separation, and is only possible by restricting to a specific kernel or kernel class.  The notation ``{$y$}'' indicates the lower bound is for a learning problem $\calP$ where the input distribution $\calD_\calX$ is fixed, and the source distributions in $\calP$ vary only in the labeling function $y(x)$ (which is deterministic in all such problems considered in the table).  This yields a ``distribution dependent'' separation (i.e.~a separation that holds even if the kernel is allowed to depend on hte input distribution).  The notation ``{\boldmath${x}$}'' indicates that the lower bound is for a learning problem $\calP$ where different source distributions $\calD\in\calP$ have varying input distributions $\calD_\calX$, but we seek a kernel that would work well for all distributions in the learning problem.  This yields a weaker ``distribution independent'' separation, as it only implies a separation if the kernel choice is not allowed to depend on the input distribution $\calD_\calX$. \info{what does this line onwards mean?}It should be noted that these differences are due to differences in how the {\em lower bound} is established, and not differences in limitations or generality of the gradient descent guarantees (in all settings considered, the gradient descent guarantees are under severe restrictions on the input distribution, but the model used for training is independent of the specific input distribution meeting these restrictions).

\paragraph{Discussion.} As can be seen from \Cref{tab:prior-sep}, in all the prior separation results, the lower bound on the error of kernel methods is bounded away from $\calL(0)$, i.e.~the upper bound on the edge is $O(1)$. In fact, in all prior separations except for \cite{daniely20parities}, the kernel error goes to zero when the error of gradient descent goes to zero, it just goes to zero slower (\citeauthor{daniely20parities} establish a lower bound on the error that is bounded away from zero, but it still corresponds to a constant edge).  In contrast, we exhibit a situation where the NTK, or  any kernel method, has only vanishing (or zero) edge, and study how quickly this edge goes to zero. In order to obtain a separation where the edge of the NTK at initialization is zero, and the edge of any kernel is exponentially small, we construct a learning problem where the input distribution is not fixed, and where gradient descent learning is possible only with a biased initialization---this is necessary otherwise Theorems \ref{thm:unbiased} and \ref{thm:dist-dependent} would tell us the edge must be at least polynomial.  It is interesting to note that none of the prior constructions are of this nature: the only prior construction that relies on variable input distributions is that of \citet{daniely20parities}, but they use an unbiased initialization.\removed{(most other constructions also use an unbiased, or nearly unbiased, initialization, or could be made to use one, and do not explicitly leverage the bias as in our Separations \ref{sep:3} and \ref{sep:4}).\nati{Actually, several are biased, so this sentence doesn't quite hold}}  

The different separation results differ in the type of kernels for which they establish lower bounds.  Many of the separation results (including our Separations \ref{sep:2} and \ref{sep:4}) establish lower bounds for any kernel corresponding to a poly-dimensional feature space, and thus for the NTK of any poly-sized model.  When considering the square loss, such lower bounds also imply lower bounds for any kernel (of any dimensionality) using poly-norm predictors, as in our Claims \ref{clm:ker-lower-bound-1} and \ref{clm:ker-lower-bound-2} \citep[also see Lemma 5 in][]{kamath20approximate}.  Since \cite{daniely20parities} considered the hinge loss, they required that {\em both} the dimension {\em and} the norm be polynomial (see \citet{kamath20approximate} for a discussion on kernel lower bounds for the squared norm versus the hinge loss.  In yet unpublished followup work, we obtain lower bounds on the hinge loss also without a norm bound).

As discussed in \autoref{sec:prelims}, in order to establish lower bounds with respect to any kernel, one necessarily needs to consider an entire learning problem (with multiple different possibly labelings) rather than a specific source distribution.  \citet{ghorbani19limitations,ghorbani19linearized} take a different approach: they restrict the kernels being considered to NTKs of a specific model, NTKs of a class of models, or general rotationally invariant kernels, and so are able to establish separations for specific source distributions (between differentiable learning and using one of these specific kernels).  These separations are more similar to our Separations \ref{sep:1} and \ref{sep:3}, which are specific to the NTK of the model being used (theirs are more general), but their separations holds only for large input dimensions $n$, while ours holds already for dimension $n=2$.  \removed{In fact, all prior separation results (as well as our Separations \ref{sep:2} and \ref{sep:4}) are stated as $n\rightarrow\infty$, and when other parameters scale in appropriate ways with the input dimension $n$ (as specified in the table, and most commonly the dimension of the kernel increasing only polynomially with $n$).  For most results, this is just a simplification to allow ``any polynomial scaling'', but the actual lower bounds hold for finite $n$ once a specific behaviour of the dimension or other parameters are plugged in (similar to our Claims \ref{clm:ker-lower-bound-1} and \ref{clm:ker-lower-bound-2}).\nati{I think this is not true for Andrea's two papers, but not sure I want to get into it}}

A deficiency of our Separations \ref{sep:3} and \ref{sep:4}, shared also by \citet{li20beyondntk,allenzhu19resnets,daniely20parities} and \citet{ghorbani19limitations}, is that we do not demonstrate a fixed learning problem where the gradient descent error goes to zero.  Instead, for every $\eps>0$, we construct a different problem where the gradient descent error is $\leq \eps$ (while the kernel error is high; close to that of the null predictor). \citet{ghorbani19linearized} and \citet{allenzhu20backward}, as well as our Separations \ref{sep:1} and \ref{sep:2}, do use learning problem where gradient descent can get arbitrarily small error.  It would be interesting to strengthen Separations \ref{sep:3} and \ref{sep:4} so that gradient descent could converge to zero error (with a fixed learning problem and model).

The separation results also differ in the form of gradient descent they analyze.  
\citet{ghorbani19limitations} and \citet{daniely20parities} analyze gradient descent or gradient flow on the population, which is the limit of GD/SGD when the number of samples and/or iterations goes to infinity (although how quickly these should increase is not analyzed).
But most analysis is for batch or stochastic gradient descent (which is what would be used in practice) with polynomial samples and iterations \citep{soltanolkotabi17learning,mei2016landscape,allenzhu19resnets}, perhaps with slight modifications or regularization \citep{li20beyondntk,allenzhu20backward}.  
Our analyses is for $\tau$-approximate gradient descent with polynomial accuracy $\tau$, which subsumes batch (or mini-batch) gradient descent with polynomially many samples (or batch size).  But our analysis is perhaps odd and unnatural in that it relies on a single step with a large stepsize, rather than allowing the number of steps to increase and the stepsize to decrease.  It should be possible, though is technically much more complicated, to extend our analysis to cover also gradient descent with any stepsize smaller than the one we use, and thus also gradient flow.  This would establish strong separation also based on a more restrictive definition of differentiable learning, which requires robustness with respect to the stepsize.

Finally, the models we use are chosen to be simple to analyze, but they are perhaps not ``natural''.  We do show how they can be implemented with a sigmoidal network, but this is a network with a very specific architecture, or alternatively a fully connected network with very specific initialization and where some of the weights are fixed rather than trained.  Showing similarly strong separations with more natural architectures and random initialization would be desirable.  The empirical demonstration in \autoref{fig:parity_experiment} indicates this should be possible (though perhaps technically involved).  \citet{allenzhu19resnets,allenzhu20backward} also use a specialized residual architecture (though perhaps not as specific as ours), while the others \citep{soltanolkotabi17learning,mei2016landscape,ghorbani19limitations,li20beyondntk,daniely20parities} use fairly benign networks and initialization.


\section{Conclusion and Discussion}\label{sec:discussion}
With the study of Neural Tangent Kernels increasing in popularity, both as an analysis and methodological approach,  it important to understand the limits of the relationship between the Tangent Kernel approximation and the true differentiable model.  Furthermore, the notion of ``gradual'' learning of deep models, where we learn progressively more complex models, or more layers\removed{ \cite{GeofHinton}}, and so the success of deep learning rests on being able to make progress even with simpler, e.g.~linear models, is an appealing approach to understanding deep learning.  Indeed, when we first asked ourselves whether the tangent kernel must always have an edge in order for gradient descent to succeed, and we sought to quantify how large this edge must be, we were guided also by understanding the ``gradual'' nature of deep learning.  We were surprised to discover that in fact, with biased initialization, deep learning can succeed even without the tangent kernel having a significant edge.  

Our results also highlight the importance of the distinction between distribution dependent and independent learning, and between biased and unbiased initialization.  The gap between distribution dependent and independent learning relates to kernel (i.e.~linear) methods inherently not being able to leverage information in $\calD_\calX$: success or failure is based on whether $y|x$ is well represented by a predictor in $\calF(K,B)$, and has little to do with the marginal over $x$.  In contrast, gradient descent on a non-linear model, could behave very differently depending on $\calD_{\calX}$, as we also see empirically in the experiment in \autoref{fig:parity_experiment}.  Perhaps even more surprising is the role of the bias of the initialization.  It might seem like a benign  property, and that we should always be able to initialize with zero, or nearly-zero predictions, or at least at $\theta_0$ that is not much worse than null, or perhaps correct for the bias as in \citet{chizat19lazy}.  But we show that at least for distribution-independent learning, this is not a benign property at all: for some problems we {\em must} use biased initialization (\autoref{sep:5}).  This observation may be of independent interest, beyond the role it plays in understanding the Neural Tangent Kernel.

The learning problems and models we used to demonstrate the separation results are artificially extreme, so as to push the separation to the limit, and allow easy analytical study.  But we believe they do capture ways in which gradient descent is more powerful than kernel methods.  In the example of \autoref{sec:separation}, gradient descent starts by selecting a few ``simple features'' (the $k\ll n$ relevant coordinates), based on simple correlations.  But unlike kernel methods, gradient descent is then able to use these features in more complex ways.  We see this happening also empirically in \autoref{fig:parity_experiment} with a straightforward two-layer ReLU network, where gradient descent is able to succeed in learning the complex parity function, once there is enough signal to easily identify the few relevant coordinates.  In the example of \autoref{sec:no-weak-learn}, we see how gradient descent is also able to pick up on structure in the input (unlabeled data) distribution, in a way that kernel methods are fundamentally unable to.


\section*{Acknowledgements}

We thank Gilad Yehudai for clarifying our questions about \citet{yehudai19power}.  This work was done while NS was visiting EPFL.\unsure{Any more acknowledgements?}  This research is part of the NSF/Simons funded {\em Collaboration on the Theoretical Foundations of Deep Learning} (\href{https://deepfoundations.ai/}{deepfoundations.ai}). PK was partially supported by NSF BIGDATA award 1546500. 

\newpage

\bibliography{main.bbl}

\begin{thebibliography}{27}
\providecommand{\natexlab}[1]{#1}
\providecommand{\url}[1]{\texttt{#1}}
\expandafter\ifx\csname urlstyle\endcsname\relax
  \providecommand{\doi}[1]{doi: #1}\else
  \providecommand{\doi}{doi: \begingroup \urlstyle{rm}\Url}\fi

\bibitem[Abbe and Sandon(2020{\natexlab{a}})]{abbe2020polytime}
E.~Abbe and C.~Sandon.
\newblock Poly-time universality and limitations of deep learning.
\newblock \emph{arXiv}, abs/2001.02992, 2020{\natexlab{a}}.
\newblock URL \url{http://arxiv.org/abs/2001.02992}.

\bibitem[Abbe and Sandon(2020{\natexlab{b}})]{abbe2020universality}
E.~Abbe and C.~Sandon.
\newblock On the universality of deep learning.
\newblock In \emph{Advances in Neural Information Processing Systems 33: Annual
  Conference on Neural Information Processing Systems 2020, NeurIPS 2020,
  December 6-12, 2020, virtual}, 2020{\natexlab{b}}.
\newblock URL
  \url{https://proceedings.neurips.cc/paper/2020/hash/e7e8f8e5982b3298c8addedf6811d500-Abstract.html}.

\bibitem[Allen{-}Zhu and Li(2019)]{allenzhu19resnets}
Z.~Allen{-}Zhu and Y.~Li.
\newblock What can resnet learn efficiently, going beyond kernels?
\newblock In \emph{Advances in Neural Information Processing Systems 32: Annual
  Conference on Neural Information Processing Systems 2019, NeurIPS 2019,
  December 8-14, 2019, Vancouver, BC, Canada}, pages 9015--9025, 2019.
\newblock URL
  \url{https://proceedings.neurips.cc/paper/2019/hash/5857d68cd9280bc98d079fa912fd6740-Abstract.html}.

\bibitem[Allen{-}Zhu and Li(2020)]{allenzhu20backward}
Z.~Allen{-}Zhu and Y.~Li.
\newblock Backward feature correction: How deep learning performs deep
  learning.
\newblock \emph{arXiv}, abs/2001.04413, 2020.
\newblock URL \url{https://arxiv.org/abs/2001.04413}.

\bibitem[Allen{-}Zhu et~al.(2019)Allen{-}Zhu, Li, and Liang]{allen2018learning}
Z.~Allen{-}Zhu, Y.~Li, and Y.~Liang.
\newblock Learning and generalization in overparameterized neural networks,
  going beyond two layers.
\newblock In \emph{Advances in Neural Information Processing Systems 32: Annual
  Conference on Neural Information Processing Systems 2019, NeurIPS 2019,
  December 8-14, 2019, Vancouver, BC, Canada}, pages 6155--6166, 2019.
\newblock URL
  \url{https://proceedings.neurips.cc/paper/2019/hash/62dad6e273d32235ae02b7d321578ee8-Abstract.html}.

\bibitem[Arora et~al.(2019{\natexlab{a}})Arora, Du, Hu, Li, Salakhutdinov, and
  Wang]{arora19exact}
S.~Arora, S.~S. Du, W.~Hu, Z.~Li, R.~Salakhutdinov, and R.~Wang.
\newblock On exact computation with an infinitely wide neural net.
\newblock In \emph{Advances in Neural Information Processing Systems 32: Annual
  Conference on Neural Information Processing Systems 2019, NeurIPS 2019,
  December 8-14, 2019, Vancouver, BC, Canada}, pages 8139--8148,
  2019{\natexlab{a}}.
\newblock URL
  \url{https://proceedings.neurips.cc/paper/2019/hash/dbc4d84bfcfe2284ba11beffb853a8c4-Abstract.html}.

\bibitem[Arora et~al.(2019{\natexlab{b}})Arora, Du, Hu, Li, and
  Wang]{arora2019fine}
S.~Arora, S.~S. Du, W.~Hu, Z.~Li, and R.~Wang.
\newblock Fine-grained analysis of optimization and generalization for
  overparameterized two-layer neural networks.
\newblock \emph{arXiv preprint arXiv:1901.08584}, 2019{\natexlab{b}}.

\bibitem[Chizat et~al.(2019{\natexlab{a}})Chizat, Oyallon, and
  Bach]{chizat19lazy}
L.~Chizat, E.~Oyallon, and F.~Bach.
\newblock On lazy training in differentiable programming.
\newblock In \emph{Advances in Neural Information Processing Systems 32: Annual
  Conference on Neural Information Processing Systems 2019, NeurIPS 2019, 8-14
  December 2019, Vancouver, BC, Canada}, pages 2933--2943, 2019{\natexlab{a}}.
\newblock URL
  \url{http://papers.nips.cc/paper/8559-on-lazy-training-in-differentiable-programming}.

\bibitem[Chizat et~al.(2019{\natexlab{b}})Chizat, Oyallon, and
  Bach]{chizat2018note}
L.~Chizat, E.~Oyallon, and F.~Bach.
\newblock On lazy training in differentiable programming.
\newblock In \emph{Advances in Neural Information Processing Systems}, pages
  2933--2943, 2019{\natexlab{b}}.

\bibitem[Daniely and Malach(2020)]{daniely20parities}
A.~Daniely and E.~Malach.
\newblock Learning parities with neural networks.
\newblock In \emph{Advances in Neural Information Processing Systems 33: Annual
  Conference on Neural Information Processing Systems 2020, NeurIPS 2020,
  December 6-12, 2020, virtual}, 2020.
\newblock URL
  \url{https://proceedings.neurips.cc/paper/2020/hash/eaae5e04a259d09af85c108fe4d7dd0c-Abstract.html}.

\bibitem[Daniely et~al.(2016)Daniely, Frostig, and Singer]{daniely16toward}
A.~Daniely, R.~Frostig, and Y.~Singer.
\newblock Toward deeper understanding of neural networks: The power of
  initialization and a dual view on expressivity.
\newblock In \emph{Advances in Neural Information Processing Systems 29: Annual
  Conference on Neural Information Processing Systems 2016, December 5-10,
  2016, Barcelona, Spain}, pages 2253--2261, 2016.
\newblock URL
  \url{https://proceedings.neurips.cc/paper/2016/hash/abea47ba24142ed16b7d8fbf2c740e0d-Abstract.html}.

\bibitem[Du et~al.(2019{\natexlab{a}})Du, Lee, Li, Wang, and Zhai]{du19global}
S.~Du, J.~Lee, H.~Li, L.~Wang, and X.~Zhai.
\newblock Gradient descent finds global minima of deep neural networks.
\newblock In K.~Chaudhuri and R.~Salakhutdinov, editors, \emph{Proceedings of
  the 36th International Conference on Machine Learning}, volume~97 of
  \emph{Proceedings of Machine Learning Research}, pages 1675--1685. PMLR,
  09--15 Jun 2019{\natexlab{a}}.
\newblock URL \url{http://proceedings.mlr.press/v97/du19c.html}.

\bibitem[Du et~al.(2019{\natexlab{b}})Du, Zhai, Poczos, and
  Singh]{du2018provably}
S.~S. Du, X.~Zhai, B.~Poczos, and A.~Singh.
\newblock Gradient descent provably optimizes over-parameterized neural
  networks.
\newblock In \emph{International Conference on Learning Representations},
  2019{\natexlab{b}}.

\bibitem[Geiger et~al.(2020)Geiger, Spigler, Jacot, and Wyart]{Geiger_2020}
M.~Geiger, S.~Spigler, A.~Jacot, and M.~Wyart.
\newblock Disentangling feature and lazy training in deep neural networks.
\newblock \emph{Journal of Statistical Mechanics: Theory and Experiment},
  2020\penalty0 (11):\penalty0 113301, nov 2020.
\newblock \doi{10.1088/1742-5468/abc4de}.
\newblock URL \url{https://doi.org/10.1088/1742-5468/abc4de}.

\bibitem[Ghorbani et~al.(2019{\natexlab{a}})Ghorbani, Mei, Misiakiewicz, and
  Montanari]{ghorbani19limitations}
B.~Ghorbani, S.~Mei, T.~Misiakiewicz, and A.~Montanari.
\newblock Limitations of lazy training of two-layers neural network.
\newblock In \emph{Advances in Neural Information Processing Systems 32: Annual
  Conference on Neural Information Processing Systems 2019, NeurIPS 2019,
  December 8-14, 2019, Vancouver, BC, Canada}, pages 9108--9118,
  2019{\natexlab{a}}.
\newblock URL
  \url{https://proceedings.neurips.cc/paper/2019/hash/c133fb1bb634af68c5088f3438848bfd-Abstract.html}.

\bibitem[Ghorbani et~al.(2019{\natexlab{b}})Ghorbani, Mei, Misiakiewicz, and
  Montanari]{ghorbani19linearized}
B.~Ghorbani, S.~Mei, T.~Misiakiewicz, and A.~Montanari.
\newblock Linearized two-layers neural networks in high dimension.
\newblock \emph{arXiv}, abs/1904.12191, 2019{\natexlab{b}}.
\newblock URL \url{http://arxiv.org/abs/1904.12191}.

\bibitem[Ghorbani et~al.(2020)Ghorbani, Mei, Misiakiewicz, and
  Montanari]{ghorbani20when}
B.~Ghorbani, S.~Mei, T.~Misiakiewicz, and A.~Montanari.
\newblock When do neural networks outperform kernel methods?
\newblock In \emph{Advances in Neural Information Processing Systems},
  volume~33, pages 14820--14830. Curran Associates, Inc., 2020.
\newblock URL
  \url{https://proceedings.neurips.cc/paper/2020/file/a9df2255ad642b923d95503b9a7958d8-Paper.pdf}.

\bibitem[Jacot et~al.(2018)Jacot, Hongler, and Gabriel]{jacot18ntk}
A.~Jacot, C.~Hongler, and F.~Gabriel.
\newblock Neural tangent kernel: Convergence and generalization in neural
  networks.
\newblock In \emph{Advances in Neural Information Processing Systems 31: Annual
  Conference on Neural Information Processing Systems 2018, NeurIPS 2018,
  December 3-8, 2018, Montr{\'{e}}al, Canada}, pages 8580--8589, 2018.
\newblock URL
  \url{https://proceedings.neurips.cc/paper/2018/hash/5a4be1fa34e62bb8a6ec6b91d2462f5a-Abstract.html}.

\bibitem[Kamath et~al.(2020)Kamath, Montasser, and Srebro]{kamath20approximate}
P.~Kamath, O.~Montasser, and N.~Srebro.
\newblock Approximate is good enough: Probabilistic variants of dimensional and
  margin complexity.
\newblock In \emph{Conference on Learning Theory, {COLT} 2020, 9-12 July 2020,
  Virtual Event [Graz, Austria]}, volume 125 of \emph{Proceedings of Machine
  Learning Research}, pages 2236--2262. {PMLR}, 2020.
\newblock URL \url{http://proceedings.mlr.press/v125/kamath20b.html}.

\bibitem[Li and Liang(2018)]{li2018learning}
Y.~Li and Y.~Liang.
\newblock Learning overparameterized neural networks via stochastic gradient
  descent on structured data.
\newblock In \emph{Advances in Neural Information Processing Systems},
  volume~31. Curran Associates, Inc., 2018.
\newblock URL
  \url{https://proceedings.neurips.cc/paper/2018/file/54fe976ba170c19ebae453679b362263-Paper.pdf}.

\bibitem[Li et~al.(2020)Li, Ma, and Zhang]{li20beyondntk}
Y.~Li, T.~Ma, and H.~R. Zhang.
\newblock Learning over-parametrized two-layer neural networks beyond {NTK}.
\newblock In \emph{Conference on Learning Theory, {COLT} 2020, 9-12 July 2020,
  Virtual Event [Graz, Austria]}, volume 125 of \emph{Proceedings of Machine
  Learning Research}, pages 2613--2682. {PMLR}, 2020.
\newblock URL \url{http://proceedings.mlr.press/v125/li20a.html}.

\bibitem[Mei et~al.(2018)Mei, Bai, and Montanari]{mei2016landscape}
S.~Mei, Y.~Bai, and A.~Montanari.
\newblock {The landscape of empirical risk for nonconvex losses}.
\newblock \emph{The Annals of Statistics}, 46\penalty0 (6A):\penalty0 2747 --
  2774, 2018.
\newblock \doi{10.1214/17-AOS1637}.
\newblock URL \url{https://doi.org/10.1214/17-AOS1637}.

\bibitem[Sch{\"{o}}lkopf and Smola(2002)]{smola1998learning}
B.~Sch{\"{o}}lkopf and A.~J. Smola.
\newblock \emph{Learning with Kernels: support vector machines, regularization,
  optimization, and beyond}.
\newblock Adaptive computation and machine learning series. {MIT} Press, 2002.
\newblock ISBN 9780262194754.
\newblock URL \url{https://www.worldcat.org/oclc/48970254}.

\bibitem[Soltanolkotabi(2017)]{soltanolkotabi17learning}
M.~Soltanolkotabi.
\newblock Learning relus via gradient descent.
\newblock In \emph{Advances in Neural Information Processing Systems 30: Annual
  Conference on Neural Information Processing Systems 2017, December 4-9, 2017,
  Long Beach, CA, {USA}}, pages 2007--2017, 2017.
\newblock URL
  \url{https://proceedings.neurips.cc/paper/2017/hash/e034fb6b66aacc1d48f445ddfb08da98-Abstract.html}.

\bibitem[Woodworth et~al.(2020)Woodworth, Gunasekar, Lee, Moroshko, Savarese,
  Golan, Soudry, and Srebro]{woodworth20kernelrich}
B.~E. Woodworth, S.~Gunasekar, J.~D. Lee, E.~Moroshko, P.~Savarese, I.~Golan,
  D.~Soudry, and N.~Srebro.
\newblock Kernel and rich regimes in overparametrized models.
\newblock In \emph{Conference on Learning Theory, {COLT} 2020, 9-12 July 2020,
  Virtual Event [Graz, Austria]}, volume 125 of \emph{Proceedings of Machine
  Learning Research}, pages 3635--3673. {PMLR}, 2020.
\newblock URL \url{http://proceedings.mlr.press/v125/woodworth20a.html}.

\bibitem[Yehudai and Shamir(2019)]{yehudai19power}
G.~Yehudai and O.~Shamir.
\newblock On the power and limitations of random features for understanding
  neural networks.
\newblock In \emph{Advances in Neural Information Processing Systems 32: Annual
  Conference on Neural Information Processing Systems 2019, NeurIPS 2019, 8-14
  December 2019, Vancouver, BC, Canada}, pages 6594--6604, 2019.
\newblock URL
  \url{http://papers.nips.cc/paper/8886-on-the-power-and-limitations-of-random-features-for-understanding-neural-networks}.

\bibitem[Zou et~al.(2020)Zou, Cao, Zhou, and Gu]{zou20gradient}
D.~Zou, Y.~Cao, D.~Zhou, and Q.~Gu.
\newblock Gradient descent optimizes over-parameterized deep relu networks.
\newblock \emph{Mach. Learn.}, 109\penalty0 (3):\penalty0 467--492, 2020.
\newblock \doi{10.1007/s10994-019-05839-6}.
\newblock URL \url{https://doi.org/10.1007/s10994-019-05839-6}.

\end{thebibliography}

\newpage
\onecolumn
\appendix

\section{Prior work on Separations between Differentiable Learning and Kernels}\label{apx:prior-separations}

We provide more details on the prior work presented in \Cref{tab:prior-sep} that demonstrate learning tasks where gradient-descent based learning can ensure smaller error than any reasonably sized kernel method.

In the summary below, whenever a $\poly(\cdot)$ terms appears under ``Learnability with Gradient Descent'', it is always for a particular $\poly(\cdot)$ that has been explicitly or implicitly specified in the corresponding reference. On the other hand, for every $\poly(\cdot)$ term that appears under ``Non-learnability with any Kernel Method'', it is always meant to hold for {\em any} $\poly(\cdot)$.

\begin{enumerate}[leftmargin=*]
\item \cite{yehudai19power,ghorbani19linearized} : \textbf{Class of ReLU neurons over Gaussian inputs}
\begin{itemize}[leftmargin=3mm]
\item {\em Learning Problem:} With $\calX=\bbR^n$ and $\calY=\bbR$, a source distribution is given by $\calD_{\calX} = \calN(0,I_n)$ and $y = [\inangle{w,x} + b]_+$. To make the result comparable to others, in terms of the value of the loss, we consider a normalization of $w$ and $b$ such that the loss of the null predictor $\calL(0) = 1/2$. In the original formulation of \cite{yehudai19power}, $w,b$ are chosen s.t. $\|w\|, |b| \le O(n^4)$, so that the value $y$ is bounded by $O(n^4)$, and the loss of the null predictor is $O(n^8)$ (due to the square loss). Therefore, in order to have $\calL(0)=\Theta(1)$, the values of $w,b$ are scaled down by $1/n^4$, and the lower-bound on the loss is $\Omega(1/n^8)$ (and not $\Omega(1)$ as stated in the paper).
\item {\em Learnability with Gradient Descent:} \cite{soltanolkotabi17learning,mei2016landscape} showed that, in the case when $b=0$, projected batch gradient descent over the ReLU model with $\poly(n/\eps)$ samples with zero initialization (i.e., unbiased initialization with $w^{(0)} = 0$), can ensure square loss at most $\eps$ with $O(\poly(n)\log(1/\eps))$ steps. Importantly, this holds for any $\eps > 0$. In the analysis of the positive result, initializing from zero is important, since in this case the first step of gradient-descent already captures the ``direction'' of the target neuron.


\item {\em Non-learnability with any Kernel Method:} \cite{yehudai19power} show that any randomized feature map of dimension at most $2^{c_1 n}$ with coefficients of the predictors bounded by $2^{c_1 n}$, for some universal constant $c_1$, must incur a square loss of at least $\Omega(1/n^8)$; this is a scaled version of their stated result, which was for $\|w\| = n^3$ and $|b| \le 6 n^4 + 1$. Subsequently, \citet{kamath20approximate} showed the same bound can be obtained without the restriction on coefficients of the predictors, leading to the lower bound as presented in the \autoref{tab:prior-sep}.   Both of these lower bounds crucially rely on a non-zero bias term $b$.

Because the upper bounds are specific to the case $b=0$ where there is no bias term, while the lower bound relies on a bias term $b\neq 0$, the analysis of \citeauthor{yehudai19power} does not establish a separation between differentiable learning and kernel methods.  It was communicated to use by Gilad Yehudai that ongoing work with Ohad Shaimr and Gal Vardi indicates that gradient descent also succeeds in the presence of a non-zero bias term, which would yield a separation.  Alternatively, in our own yet unpublished work, we observed it is also possible to obtain a lower bound for ReLU without the bias term, which also yields a separation.\improve{fix citation when we have it.}

\cite{ghorbani19linearized} show that for any constant $c_2$, any kernel method that is either (i) a rotationally invariant kernel using at most $n^{c_2}$ samples or (ii) NTK of a depth-$2$ ReLU Net with at most $n^{c_2}$ units must incur a square loss of at least $\eps'(c_2)$, which is a positive constant that depends only on $c_2$. This result holds even for a fixed source distribution $\calD$, that is, a fixed choice of $w$ and $\calD_{\calX} = \calN(0,I_n)$. This result holds with probability approaching $1$, as the dimension $n$ and the number of features grow to infinity.
\end{itemize}
\item \cite{ghorbani19limitations} : \textbf{\boldmath (i) Convex Quadratic Functions, (ii) Mixture of Gaussians (Binary Labels)}
\begin{itemize}[leftmargin=3mm]
\item {\em Learning Problems:} Two learning problems are considered here; both w.r.t. square loss.

(i) With $\calX = \bbR^n$ and $\calY = \bbR$, a source distribution is given by Gaussian input marginals $\calD_{\calX} = \calN(0,I_n)$ and $y = b_0 + x^TBx$ for $B\succeq 0$.

(ii) With $\calX = \bbR^n$ and $\calY = \sbit$, a source distribution is specified by $(\Sigma^{(1)}, \Sigma^{(-1)})$, where $y \sim \calU(\sbit)$, and $x \mid y \sim \calN(0, \Sigma^{(y)})$.

\item {\em Learnability with Gradient Descent:} For a depth-$2$ network with quadratic activations, with $N$ units and any random initialization that is absolutely continuous w.r.t. Lebesgue measure (in particular, the initialization can be chosen to nearly unbiased), it is shown that in the asymptotic regime as $n, N \to \infty$ with $\rho = N/n < 1$, gradient flow w.r.t. population loss ensures square loss that approaches $\eps(\rho)$, which is a closed form expression that also depends on the specific source distribution (i.e.~the quadratic target or the covariances of the classes), or rather sequence of source distributions as $n\rightarrow\infty$. The requirement on the initialization is needed since there is a measure-zero set of initializations that converge to a saddle point. Depending on the problem, the saddle points can be escaped also from an unbiased initialization.
\item {\em Non-learnability with any Kernel Method:} It is shown in the same asymptotic regime of $n, N \to \infty$ with $\rho = N/n < 1$, that the NTK at initialization of the same model incurs a square loss of $\eps'(\rho)$, which is a closed form expression that also depends on the sequence of source distributions as $n\rightarrow\infty$. It is shown that for non-degenerate distribution sequences, and any $\rho<1$, that $\eps'(\rho) > \eps(\rho)$.
\end{itemize}
\item \cite{li20beyondntk} : \textbf{\boldmath A depth-$2$ Net, with abs. value activations and restricted weights, over Gaussian inputs}
\begin{itemize}[leftmargin=3mm]
\item {\em Learning Problem:} With $\calX = \bbR^n$ and $\calY = \bbR$, a source distribution is given by $\calD_{\calX} = \calN(0,I_n)$ and $y = \sum_{i=1}^n a_i \abs{\inangle{w_i, x}}$, which is a two layer network with absolute value activations, and additional constraints that $w_i$'s are orthonormal and $a_i \in [\frac{1}{\kappa n}, \frac{\kappa}{n}]$ for some $\kappa > 1$.
\item {\em Learnability with Gradient Descent:} It is shown that there exists a depth-$2$ ReLU network with $\poly(n)$ units and a random normal initialization that is not unbiased, a type of truncated batch gradient descent (where large coordinates of the gradient are set to zero) using $\poly(n)$ samples, is shown to ensure square loss at most $\eps := 1/n^{1+\delta}$ for a constant $\delta \in (0, 0.01)$ (that depends on $\kappa$) with $\poly(n)$ steps. Notably, this procedure is {\em not} shown to ensure an arbitrarily small error.
\item {\em Non-learnability with any Kernel Method:} It is shown that any kernel method with at most $\poly(n)$ features (for any $\poly(n)$) must incur a square loss of at least $\Omega(1/n) \ge \eps^{1-\delta/2}$.
\end{itemize}
\item \cite{allenzhu19resnets} : \textbf{\boldmath Sparse Variable Selection + Parity (implementable as a depth-$2$ ResNet})
\begin{itemize}[leftmargin=3mm]
\item {\em Learning Problem:} With $\calX = \bbR^n$ and $\calY = \bbR^k$ (for $k= \Theta(\log n)$), a source distribution is given by $\calD_{\calX} = \calU\inparen{\sbit^n}$ is uniform over Boolean hypercube and $y = Wx + \alpha \calG(Wx)$, where $W = (e_{i_1}, \ldots, e_{i_k}) \in \bbR^{k \times n}$ for $i_1, \ldots, i_k \in [n]$ and $\calG(z) := \prod_{i=1}^k z_i \cdot \mathbf{1}$ for $\alpha = 1/2^{\widetilde{\Theta}(k)} = 1/\poly(n)$. We can interpret this as $Wx$ selecting a $k$ sized subset of input coordinates, and $\calG$ computes the parity of that subset, in each of the $k$ coordinates.
\item {\em Learnability with Gradient Descent:} For an overparameterized depth-$2$ ResNet with $\poly(n/\eps)$ units and a random normal initialization (that has a high variance, leading to a biased initialization), stochastic gradient descent is shown to ensure square loss of $\eps \le \alpha^{3.9}$ with $\poly(n,1/\eps)$ steps. Notably, it is {\em not} shown to ensure an arbitrarily small error.
\item {\em Non-learnability with any Kernel Method:} It is shown that any kernel method with at most $\poly(n)$ (for any $\poly(n)$) features must incur a square loss of at least $\Omega(\eps^{0.52}) \ge \Omega(\alpha^2)$.\\[-3mm]

{\em Remark:} Note that the learning problem here depends on the choice of $\eps$. That is, for every $n$ and $\eps > 0$, there is a learning problem where stochastic gradient descent ensures square loss $\le \eps$ whereas no kernel method of dimension at most $\poly(n)$ can achieve a loss smaller than $\eps^{0.52}$.
\end{itemize}
\item \cite{allenzhu20backward} : \textbf{\boldmath A depth-$L$ Net with quadratic activations}
\begin{itemize}[leftmargin=3mm]
\item {\em Learning Problem:} With $\calX = \bbR^n$ and $\calY = \bbR$, a source distribution is given by $\calD_{\calX} = \calN(0,I_n)$ and $y = N_{\theta}(x)$, where $N_{\theta}$ is a depth-$L$ network with quadratic activations that has additional constraints on presence of skip connections, as well as constraints on the weights, that is referred to as the ``information gap''. $L$ can be taken to be any parameter that is $o(\log \log n)$.
\item {\em Learnability with Gradient Descent:} For an overparameterized depth-$L$ Net with quadratic activations, with identical layer structure as $N_{\theta}$, with $\poly(n)$ units and zero initialization (unbiased), regularized stochastic gradient descent (with a non-standard regularizer) is shown to ensure square loss $\eps$, with $\poly(n/\eps)$ number of steps. Importantly, this holds for any $\eps > 0$.
\item {\em Non-learnability with any Kernel Method:} It is shown that any kernel method of dimension at most $\poly(n)$ (for any $\poly(n)$) must incur a square loss of at least $\Omega(1/n^{0.01})$.
\end{itemize}
\item \cite{daniely20parities} : \textbf{Sparse parities over a mixture of uniform and leaky marginals}
\begin{itemize}[leftmargin=3mm]
\item {\em Learning Problem:} With $\calX= \sbit^n$ and $\calY = \sbit$, the learning problem is similar to $\Plp[n,\alpha]$ that we consider in \autoref{sec:no-weak-learn} with $\alpha = 1/2$. There are two differences: (i) The label under $\calD_I^{(1)}$ is given by the parity, instead of being random and (ii) the sparsity of the parity is fixed to be $k$, which can be taken to be $k=\Theta(\sqrt[10]{n})$.
\item {\em Learnability with Gradient Descent:} For a particular $\eps = 1/\poly(n)$, a depth-$2$ net with ReLU6 activation with $\poly(n)$ units and an unbiased initialization (obtained by pairing the units with complimentary initialization, as suggested by \citet{chizat19lazy}), a regularized gradient descent (standard $\ell_2$ regularization) over the {\em population loss} is shown to ensure hinge loss of $\eps$, with $\poly(n)$ steps. Notably, it is {\em not} shown to ensure an arbitrarily small error.
\item {\em Non-learnability with any Kernel Method:} It is shown that any kernel method of dimension at most $\poly(n)$ and norm at most $\poly(n)$ must incur a hinge loss of at least $1/3$ (recall that hinge loss is $\ell(\what{y},y) := \max\set{1 - \what{y}y, 0}$; and so that $\calL_{\calD}(0) = 1$ for any distribution $\calD$ with $\pm 1$ labels).  In yet unpublished work extending \citet{kamath20approximate}, we show it is possible to obtain a lower bound with only the restriction on the norm (i.e. without restricting the dimension).\\[-3mm]

\end{itemize}
\end{enumerate}



\section{Proofs of Gradient Based Learning}\label{apx:neural-net-ub}

\subsection{Proof from \autoref{sec:separation}}\label{apx:neural-net-ub-1}

\gdubone*
\begin{proof}[Proof of \autoref{clm:gd-upper-bound-1}]
We only rely on \autoref{eq:f-theta-sec3}, which gives us that at $\theta^{(0)} = 0$, we have for all $x \in \sbit^n$ that
\begin{align}
f_{\theta^{(0)}}(x) &= 0\label{eqn:unbiased-prop}\\
\nabla_{\theta} f_{\theta^{(0)}}(x)&= 2x\label{eqn:grad-at-zero}
\end{align}
and hence $\nabla_{\theta} \calL_{\calD_I}(f_{\theta^{(0)}})$ is given by
\begin{align*}
\nabla_\theta \calL_{\calD_I}(f_{\theta^{(0)}})
&~=~ \nabla_\theta \insquare{\Ex_{(x, y) \sim \calD_I}\frac{1}{2}(f_{\theta^{(0)}}(x) - y)^2}
~=~ - \Ex_{(x, y) \sim \calD_{I}} y \cdot \nabla_\theta f_{\theta^{(0)}}(x)\\
&~=~ - 2\Ex_{(x, y) \sim \calD_I} y \cdot x
~=~ - 2\Ex_{x \sim \calD_{\calX}} \left(\prod_{i \in I} x_i\right) \cdot x.
\end{align*}
Thus, for all $j \in I$ we have
\begin{align*}
\frac{\partial}{\partial \theta_j} \calL_{\calD_I}(f_{\theta^{(0)}})~=~ -2 \Ex_{x \sim \calD_\calX} \prod_{i \in I \setminus \set{j}} x_i  ~=~ -2 \alpha \prod_{i \in I \setminus \set{j}}\Ex_{x \sim \calD_1} x_i~=~ - \frac{2\alpha}{2^{k-1}} ~=~ - \frac{4\alpha}{2^{k}}\,, 
\end{align*}
and for all $j \notin I$ we have:
\begin{align*}
\frac{\partial}{\partial \theta_j} \calL_{\calD_I}(f_{\theta^{(0)}}) &~=~ - 2\Ex_{x \sim \calD_\calX} \prod_{i \in I \cup \{j\}} x_i
~=~ - \frac{2\alpha}{2^{k+1}} ~=~ - \frac{\alpha}{2^{k}}\,.
\end{align*}
With step size $\eta=2^k/(\alpha n)$, we have that after the first gradient step, (i) for all $j \in I$, $\theta_j^{(1)} \in [\frac{2^k}{\alpha n} \cdot (\frac{4\alpha}{2^k} \pm \tau)] = [\frac3n, \frac5n]$, (ii) for all $j \notin I$, $\theta_j^{(1)} \in [\frac{2^k}{\alpha n} \cdot (\frac{\alpha}{2^k} \pm \tau)] = [0,\frac2n]$ (since $\tau = \alpha/2^{k}$).
From \autoref{eq:f-theta-sec3} we get $f_{\theta^{(1)}}(x) = \prod_{i \in I} x_i$, thereby concluding the proof. 
\end{proof}

\subsection{Proof from \autoref{sec:no-weak-learn}}\label{apx:neural-net-ub-2}

\gdubtwo*
\begin{proof}[Proof of \autoref{clm:gd-upper-bound-2}]
We only rely on \autoref{eq:f-theta-sec5}, which gives us that at $\theta^{(0)} = 0$, we have for all $x \in \sbit^n$ that
\begin{align}
f_{\theta^{(0)}}(x) &~=~ -1\\
\textstyle \nabla_{\theta} f_{\theta^{(0)}}(x)&\textstyle~=~ 2 \cdot (x + \frac53\alpha \mathbf{1})
\end{align}
and hence $\nabla_{\theta} \calL_{\calD_I}(f_{\theta^{(0)}})$ is given by
\begin{align*}
\nabla_\theta \calL_{\calD_I}(f_{\theta^{(0)}})
&~=~ \Ex_{(x,y) \sim \calD_I}(f_{\theta^{(0)}}(x) - y) \cdot \nabla_\theta f_{\theta^{(0)}}(x) \\
&~=~ 2 \cdot \Ex_{(x,y) \sim \calD_I} (-1-y) \cdot (x + {\textstyle \frac53} \alpha \mathbf{1})\\
&~=~ 2(1-\alpha) \underbrace{\E_{(x, y) \sim \calD_I^{(0)}} (-1-y) x}_{=\ 0}
~+~ 2\alpha \underbrace{\E_{(x, y) \sim \calD_I^{(1)}} (-1-y) x}_{=\ -x^I} 
~+~ \underbrace{\Ex_{(x,y) \sim \calD_I}(-1 - y)}_{=\ -1} \cdot {\textstyle \frac53} \cdot 2\alpha \mathbf{1}\\
&~=~ -2\alpha (x^I + {\textstyle \frac53}\mathbf{1})
\end{align*}
where we get that $\E_{(x, y) \sim \calD_I^{(0)}} x = 0 = \E_{(x, y) \sim \calD_I^{(0)}} x \cdot y$, since $|I| \ge 2$.

With gradient accuracy of $\tau = \frac43\alpha$, we get that any valid gradient estimate $g$ must satisfy $g_i \in [-\frac{12\alpha}{3}, -\frac{20\alpha}{3}]$ for $i \in I$, and $g_i \in [0,-\frac{8\alpha}{3}]$ for $i \notin I$. Using step size $\eta=\frac{3}{4\alpha n}$, we get

\[\textstyle
\theta_i^{(1)} \in [\frac3n, \frac5n]\ \ \text{ for } i \in I
\qquad \text{ and } \qquad
\theta_i^{(1)} \in [0, \frac2n]\ \ \text{ for } i \notin I\]
From \autoref{eq:f-theta-sec5}, we get that $f_{\theta^{(1)}}(x) = \prod_{i \in I} x_i$ and it is easy to see that $\calL_{\calD_I}(f_{\theta^{(1)}}) = \alpha$.
\end{proof}

\section{Lower Bounds on Kernel Methods}\label{apx:kernel-lb}

\subsection{Lower Bounds for Tangent Kernels}\label{apx:tang-kernel-lb}

\subsubsection{NTK Edge in \texorpdfstring{\autoref{sec:separation}}{Section~\ref{sec:separation}}}\label{apx:tang-kernel-lb-1}

\ntkedge*
\begin{proof}
It follows from \eqref{eqn:unbiased-prop} and \eqref{eqn:grad-at-zero} that the tangent feature map is $\phi_{\theta_0}(x)=x$, establishing that the NTK is the standard linear kernel, and that predictors in the NTK are just linear in $x$.  For uniform inputs (i.e.~from $\calD_0$), no linear predictor $h_w(x)=\inangle{w,x}$ is correlated with the parity if $k \geq 2$ bits, and so the expected square loss of any such predictor on $\calD_0$ will be at least $\frac{1}{2}$, yielding $\calL_{\calD_I}(h_w)\geq (1-\alpha)\frac{1}{2} = \frac{1}{2}-\frac{\alpha}{2}$.

To obtain an upper bound on the error (lower bound on the edge), consider a linear predictor $h(x)=c\cdot Z$ where $Z=\sum_{i\in I} x_i$.  By calculating $\Ex[Z^2]=k+\alpha(k^2-k)/4$ and $\Ex[Z y] = \alpha k / 2^{k-1}$, and verifying the optimal scaling is $c=\Ex[Z y]/E[Z^2]$, we get that with this value of $c$, $\calL_{\calD_I}(h) = \frac{1}{2} - \gamma$, where
\begin{align}
\gamma
~=~ \frac{\Ex[Zy]^2}{2\Ex[Z^2]}
~=~  \frac{\alpha^2}{2^{2k}} \cdot \frac{8}{\alpha + (4-\alpha)/k} 
~\geq~ \frac{\alpha^2}{2^{2k}} \cdot \frac{8}{3}
~\geq~ \tau^2.
\end{align}
The actual optimal linear predictor is of the form $h(x) = c \sum_{i \in I} x_i + b \sum_{i \not\in I} x_i$, and has a slightly better edge.
\end{proof}

\subsubsection{NTK Edge in \texorpdfstring{\autoref{sec:no-weak-learn}}{Section~\ref{sec:no-weak-learn}}}\label{apx:tang-kernel-lb-2}

\ntkedgetwo*
\begin{proof}
This follows from noting that for any $h_{w,b} \in \NTK_{\theta_0}^f(B)$ of the form $h_{w,b}(x) = \inangle{w,x}+b$, it holds that $\E_{(x,y)\sim\calD_I} h_{w,b}(x) \cdot y = 0$. Therefore, the optimal scaling of $h_{w,b}$ is $0$, which incurs a loss $\calL_{\calD_I}(0) = \frac{1}{2}$.
\end{proof}

\subsection{Lower Bounds for General Kernels}

In order to prove lower bounds against general kernels, we recall the lower bounds proved on the dimension or norm of a (probabilistic) feature map needed to be able to represent certain hypothesis classes that were proved by \cite{kamath20approximate}, via probabilistic variants of dimensional and margin complexity.

Consider a ``fixed-marginal'' learning problem $\calP$ over $\calX \times \calY$ with $\calY \subseteq \bbR$, where each $\calD \in \calP$ corresponds to a hypotheses in a class $\calH \subseteq (\calX \to \calY)$. Namely, sampling $(x,y) \sim \calD_h$ (for any $h \in \calH$) is equivalent to sampling $x \sim \calD_{\calX}$ (for a fixed marginal $\calD_{\calX}$) and setting $y = h(x)$. We say that $\calP$ is ``orthogonal'' if $\Ex_{x \sim \calD_{\calX}} h(x) h'(x) = 0$ for all $h, h' \in \calH$, and we say $\calP$ is ``normalized'' if $\Ex_{x \sim \calD_{\calX}} h(x)^2 = 1$ for all $h \in \calH$.

\begin{theorem}\label{thm:dc-mc-lb}[\citet{kamath20approximate}]
Consider a fixed-marginal, orthogonal and normalized learning problem $\calP$. For any probabilistic kernel $K$ corresponding to $p$-dimensional feature maps, if for all $\calD \in \calP$ it holds for $\sqloss$ that
\[ \Ex_{K}~\inf_{h \in \calF(K,B)}~\calL_\calD(h) \le \frac{1}{2} - \gamma \]
then, it holds that (i) $p \ge 2\gamma \cdot |\calP|$ and (ii) $B^2 \ge \Omega(\gamma^3 \cdot |\calP|)$.\unsure{something is off in the units...}

Equivalently, we have that there exists $\calD \in \calP$ such that
\[ \Ex_{K}~\inf_{h \in \calF(K,B)}~\calL_\calD(h) \ge \frac{1}{2} - \min\set{\frac{p}{2|\calP|}, O\inparen{\frac{B^{\nicefrac{2}{3}}}{|\calP|^{\nicefrac{1}{3}}}}} \]
\end{theorem}
\begin{proof}
Part (i). The claim of $p \ge 2\gamma \cdot |\calP|$ follows immediately from Theorem 19 in \cite{kamath20approximate}.

Part (ii). It follows from Lemma~5 in \cite{kamath20approximate} that irrespective of $p$, it is possible to construct a new randomized kernel $K'$ corresponding to a feature map $\phi' : \calX \to \bbR^{p'}$ with $p' \le B^2 \cdot O\inparen{\frac{\frac12 - \gamma + \eta}{\eta^2}}$ for any choice of $\eta > 0$, to get that
\[ \Ex_{K'}~\inf_{h \in \calF(K',B)}~\calL_\calD(h) \le \frac{1}{2} - \gamma + \eta \]
We set $\eta = \gamma/2$, so that we can take $p' \le B^2 \cdot O(1/\gamma^2)$. From Part (i), we now have that $p' \ge 2 (\gamma - \eta) \cdot |\calP| = \gamma \cdot |\calP|$. Putting everything together, we get that $B^2 \ge \Omega(\gamma^3 \cdot |\calP|)$.
\end{proof}

We now use \autoref{thm:dc-mc-lb} to prove lower bounds in Sections~\ref{sec:separation} and \ref{sec:no-weak-learn}. The key idea is that while, the learning problems considered are either not orthogonal, or do not have fixed marginals, they contain a large component that is a learning problem with fixed marginals that is orthogonal.

\subsubsection{Proof from \autoref{sec:separation}}\label{apx:kernel-lb-2}

\kerlbone*
\begin{proof}[Proof of \autoref{clm:ker-lower-bound-1}]
Fix a probabilistic kernel $K$ corresponding to $p$-dimensional feature maps. For all $\calD_I \in \calP$, we have
\begin{align}
\Ex_{K} ~ \inf_{h \in \calF(K,B)}~\calL_{\calD_I}(h)
&~\ge~ (1-\alpha) \cdot \Ex_{K} ~ \inf_{h \in \calF(K,B)}~\calL_{\calD_I^{(0)}}(h)
~+~ \alpha \cdot \Ex_{K} ~ \inf_{h \in \calF(K,B)}~\calL_{\calD_I^{(1)}}(h)\label{eq:ker-lb-1-1}
\end{align}
where $\calD_I^{(b)}$ corresponds to sampling $x \sim \calD_b$ and setting $y = \chi_I(x)$. We note that the second term, involving $\calD_I^{(0)}$ is non-negative. To lower bound the first term, observe that the learning problem $\calP^{(0)}$, consisting of distributions $\calD_I^{(0)}$, has a fixed-marginal, and is orthogonal and normalized.
Thus from \autoref{thm:dc-mc-lb} it follows that
\begin{align*}
\Ex_{K}~\inf_{h \in \calF(K,B)}~\calL_{\calD_I^{(0)}}(h)
&~\ge~ \frac{1}{2} - \min\set{\frac{p}{2|\calP|}, O\inparen{\frac{B^{\nicefrac{2}{3}}}{|\calP|^{\nicefrac{1}{3}}}}}
\end{align*}
Combining with \eqref{eq:ker-lb-1-1}, we get
\[
\Ex_{K}~\inf_{h \in \calF(K,B)}~\calL_{\calD_I}(h) ~\ge~ (1-\alpha) \cdot \inparen{\frac{1}{2} - \min\set{\frac{p}{2|\calP|}, O\inparen{\frac{B^{\nicefrac{2}{3}}}{|\calP|^{\nicefrac{1}{3}}}}}} ~\ge~ \frac{1}{2} - \frac{\alpha}{2} - \min\set{\frac{p}{2|\calP|}, O\inparen{\frac{B^{\nicefrac{2}{3}}}{|\calP|^{\nicefrac{1}{3}}}}}
\]
This completes the proof by noting that $\calL_{\calD_I}(0) = 1/2$ for all $\calD_I \in \calP$.
\end{proof}

\begin{remark}\label{remove_mix}
One can also directly bound the optimal edge of a kernel with feature map $\phi^p=(\phi_1,\dots, \phi_p)$ when the input distribution is i.i.d.\ of bias $b=1/2$, without relying on the mixture component. Denoting by $\mathcal{D}_1^{(b)}$ the production measure with $\mathbb{E}x_i=b$, the optimal edge of a $p$-dimensional kernel 
$$\max_{\phi_1,\dots,\phi_p} \min_{I \subseteq [n]: |I|=k}  \|\chi_I/\langle \phi^p \rangle  \|_{L_2(\mathcal{D}_1^{(b)})}^2,$$ can be upper-bounded by splitting the contribution of features of the degree $\le d$ and degree $>d$ polynomials, and optimizing over $d$. More specifically, defining 
the Fourier Walsh basis $\{\tilde{\chi}_I\}_{I \subseteq [n]}$ for the biased measure $\mathcal{D}_1^{(b)}$, one can decompose for $d \le k$ the features by projecting them on $\{\tilde{\chi}_I\}_{I \subseteq [n]: |I| \le d}$ and $\{\tilde{\chi}_I\}_{I \subseteq [n]: |I| > d}$. One 
can then crudely bound the first contribution of the edge by taking the best approximation in $\{\tilde{\chi}_I\}_{I \subseteq [n]: |I| \le d}$, ignoring the dimension-$p$ constraint of the features, which gives a bound of  $\sum_{t=0}^d \binom{k}{t} (1-b^2)^t b^{2(k-t)}$. The second contribution can then be upper-bounded with techniques similar as to the previous section; with a bound that is multiplicative in $p$ but exponential in $d\log(k/n)$, due to the high-degree contribution. Choosing $d$ appropriately, such $d=\sqrt{\log(n)}$, then gives a vanishing edge (for both contributions and in total). The mixture distribution $\mathcal{D}_I$ allows instead for the more direct argument based on \cite{kamath20approximate}. 
\end{remark}

\subsubsection{Proof from \autoref{sec:no-weak-learn}}\label{apx:kernel-lb-3}
\kerlbtwo*
\begin{proof}[Proof of \autoref{clm:ker-lower-bound-2}]
Fix a probabilistic kernel $K$ corresponding to $p$-dimensional feature maps.
For all $\calD_I \in \calP$, we have
\begin{align*}
\Ex_{K} ~ \inf_{h \in \calF(K,B)}~\calL_{\calD_I}(h) 
&~\ge~ (1-\alpha) \cdot \Ex_{K} ~ \inf_{h \in \calF(K,B)}~\calL_{\calD_I^{(0)}}(h)
+ \alpha \cdot \Ex_{K} ~ \inf_{h \in \calF(K,B)}~\calL_{\calD_I^{(1)}}(h)
\end{align*}
Note that $\calL_{\calD_I^{(1)}}(h) ~\ge~ 1/2$ for any $h$.
Now, the learning problem $\calP^{(0)}$, consisting of distributions $\calD_I^{(0)}$, has a fixed-marginal, and is orthogonal and normalized.
Thus we have from \autoref{thm:dc-mc-lb} that
\begin{align*}
\Ex_{K}~\inf_{h \in \calF(K,B)}~\calL_{\calD_I^{(0)}}(h)
&~\ge~ \frac{1}{2} - \min\set{\frac{p}{2|\calP|}, O\inparen{\frac{B^{\nicefrac{2}{3}}}{|\calP|^{\nicefrac{1}{3}}}}}
\end{align*}
This completes the proof by noting that $\calL_{\calD_I}(0) = 1/2$ for all $\calD_I \in \calP$.
\end{proof}



\end{document}